%% file: main.tex
\documentclass{svproc}
\usepackage{url}

\usepackage[numbers]{natbib}
\usepackage{mathrsfs}
\usepackage{color}
\usepackage{times}
\usepackage{epsfig}
\usepackage{graphicx}
\usepackage{tabularx}
\usepackage[ruled,vlined,linesnumbered]{algorithm2e}
\usepackage{graphicx} 
\usepackage{tikz}
\usepackage{verbatim}
\usepackage{booktabs}
\usepackage{multirow}
\usepackage{hyphenat} 
\usepackage{pgf}
\usetikzlibrary{arrows,automata}
\usepackage{url}
\usepackage{thmtools}
\usepackage{color}
\usepackage{times}
\usepackage{epsfig}
\usepackage{graphicx}
\usepackage{tabularx}
\usepackage{amsmath}
\usepackage{amssymb}
\usepackage[ruled,vlined,linesnumbered]{algorithm2e}
\usepackage{graphicx} 
\usepackage{tikz}
\usepackage{booktabs}
\usepackage{multirow}
\usepackage{makecell}
\usepackage{authblk}
\usepackage{hyphenat} 
\usepackage{pgf}
\usetikzlibrary{arrows,automata}

\DeclareMathOperator*{\argmin}{arg\,min}

\newcommand{\AD}[1]{{\color{red} {\bf AD:} #1}}
\newcommand{\AM}[1]{{\color{blue} {\bf AM:} #1}}

\usepackage{hyperref}

\begin{document}

\mainmatter              
\title{Energy-Aware Coverage Planning for Heterogeneous Multi-Robot System}
\titlerunning{Energy-Aware Coverage}  
%
\author{Aiman Munir$^{1}$ \and Ayan Dutta$^{2}$  \and  Ramviyas Parasuraman$^{1}$ }
\authorrunning{Munir et al.} 
%
\tocauthor{Aiman Munir, Ayan Dutta, Ramviyas Parasuraman}

\institute{$^{1}$ School of Computing, University of Georgia, Athens, GA 30602, USA,\\
$^{2}$ School of Computing, University of North Florida, Jacksonville, FL 32224, USA,\\
Corresponding author email: \email{ramviyas@uga.edu},\\ 
Codes, Videos, and Appendix: \url{https://github.com/herolab-uga/energy-aware-coverage}
}

\maketitle              

\begin{abstract}
We propose a distributed control law for a heterogeneous multi-robot coverage problem, where the robots could have different energy characteristics, such as capacity and depletion rates, due to their varying sizes, speeds, capabilities, and payloads. Existing energy-aware coverage control laws consider capacity differences but assume the battery depletion rate to be the same for all robots. In realistic scenarios, however, some robots can consume energy much faster than other robots; for instance, UAVs hover at different altitudes, and these changes could be dynamically updated based on their assigned tasks. Robots' energy capacities and depletion rates need to be considered to maximize the performance of a multi-robot system. To this end, we propose a new energy-aware controller based on Lloyd's algorithm to adapt the weights of the robots based on their energy dynamics and divide the area of interest among the robots accordingly. The controller is theoretically analyzed and extensively evaluated through simulations and real-world demonstrations in multiple realistic scenarios and compared with three baseline control laws to validate its performance and efficacy.
\end{abstract}

\section{Introduction}
    
There has been an increase in the use of autonomous robots in recent years, particularly for the purpose of surveillance and monitoring environments. 
The collected data from the robots can be used to make further decisions. For example, in a precision agriculture application, robot-collected hyper-spectral images might be used for weed localization and treating the affected areas with herbicides~\cite{tokekar2016sensor}. 
For such applications, sensor (or spatial) coverage is an important computational problem to consider. In a multi-robot sensor coverage problem, the objective is to distribute the robots (sensors) in a manner that optimally monitors the workspace covering the spatial region (environment) with at least one robot's sensor footprint. This objective is different from and is not to be confused with the area coverage problem, where the objective is that the robot(s) should visit each and every part of the environment. 
A promising technique for achieving the sensor coverage goal is to divide the environment into regions using Lloyd's algorithm, also known as Voronoi partitions, which can be optimized based on constraints such as uncertainties in sensing and sensor health, among others~\cite{santos2018coverage,shi2020,pierson2016adaptivetrustweighting,arslan2016}. 


A heterogeneous group of robots (e.g., a group of UGVs and UAVs) allows the system to develop proficiency in different areas of the task despite their inherent limitations on a specific capability, such as sensing or mobility \cite{rudolph2021range,cai2021non}. UAVs, for instance, can cover more terrain and are less susceptible to obstacles than UGVs, which, on the other hand, generally have higher battery capacities and deplete energy at a slower rate than UAVs~\cite{yu2019coverage}. 
We posit that the robot's energy depletion rate is an important factor that needs to be incorporated into their controllers to increase the overall lifetime of the multi-robot system and optimize the mission objective.

\begin{figure}[t]
    \centering
    \begin{tabular}{cc}
        \hspace{-0.3in}\includegraphics[width=0.45\linewidth]{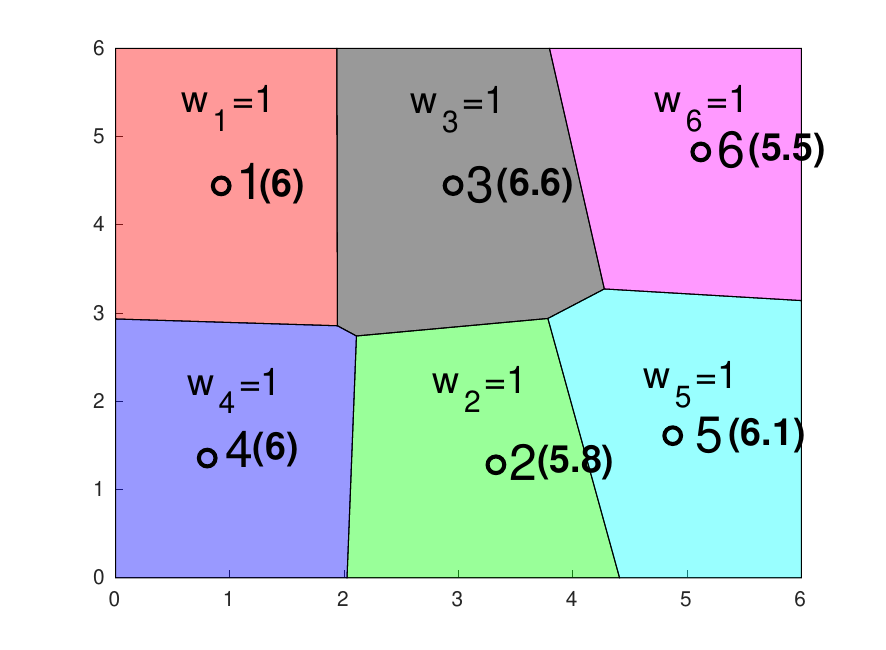}&
        \hspace{-0.1in}\includegraphics[width=0.45\linewidth]{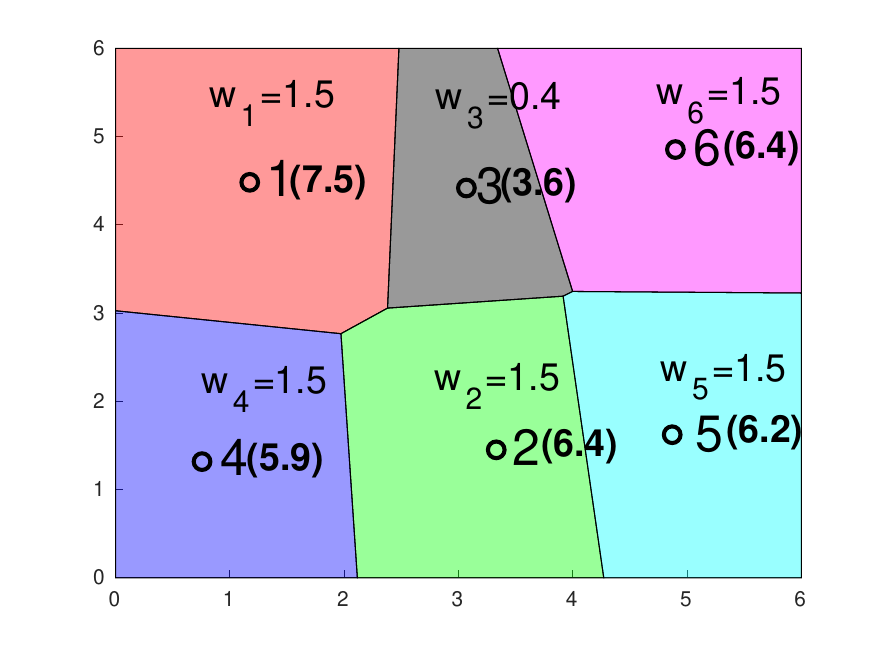} 
    \end{tabular}
    \vspace{-4mm}
    \caption{Regions assigned by standard Voronoi partitioning (left) and the proposed energy-aware controller (EAC) (right). The robots have the same initial battery level. However, robot 3's depletes its energy three times faster than the other robots. Therefore, robot 3's EAC-assigned region area (area in $m^2$ in parentheses) is less than other robots by adapting the weights ($w_i$) based on the ratio of energy depletion rate between the robots.
    }
    \label{fig:illust}
    \vspace{-4mm}
\end{figure}

The rate of energy depletion could vary significantly between robots due to their sensor payloads (e.g., some robots may be equipped with high-power thermal cameras or power drill tools in search and rescue applications), as well as their velocities (e.g., UAVs operating at higher altitude will deplete energy faster than the ones at lower altitudes) \cite{parasuraman2014model,mei2004energy,wei2020energy}. 
Moreover, the energy depletion rate of a robot can change dynamically over time (e.g., activating or deactivating a payload, increasing the velocity, etc.), and the robots need to adjust their coverage loads accordingly. Fig.~\ref{fig:illust} illustrates such a scenario. 
Here, the energy-aware controller optimally assigns a lower area to robot 3, which had almost four times higher energy depletion rate compared to the other robots, which will limit its ability to cover more area. 
Assigning a larger area to a robot with more energy or a lower depletion rate and smaller areas to other robots will distribute the workload among robots and ensure effective coverage and timely execution of tasks. Energy availability has significant implications for coverage planning. Several works proposed in the literature \cite{kim2022,kwok2007energy,derenick2011energy,duca2020multi} focusing on this problem consider limiting the robot velocity and/or partitioning the regions based on the current energy levels. However, this will result in poor coverage quality when robots have different energy depletion rates. 
Moreover, as discussed earlier, real-world applications might require a heterogeneous group of robots with different and dynamic energy characteristics, where robots consume energy differently.
Moreover, restricting the robots' velocities like performed in \cite{kwok2007energy,duca2020multi} could conflict with the low-level, time-limited tasks the robots are assigned to.

To this end, we present a novel distributed multi-robot controller that holistically considers the robot's energy capacity and depletion characteristics of other robots for optimal energy-aware coverage planning. 
We extensively validate the controller in simulations (scaling up to $100$ robots) and with real robots within an in-house swarm robotics testbed (see the video). The main contributions of our paper are as follows.
\begin{itemize}
    \item To the best of our knowledge, this is the first study that considers differences in the energy capability of the robots with heterogeneous energy depletion rates for effective multi-robot coverage planning.
    \item Unlike prior studies, we consider that the robots might have time-varying energy depletion rates during a mission. Compared to the state-of-the-art relevant algorithms, our proposed Energy-Aware Controller (EAC) significantly reduces coverage cost and achieves energy-balancing coverage objectives. 
\end{itemize}

\section{Related Work}
Sensor coverage refers to the optimal placement of robots (or sensors) in an environment that ensures maximum coverage such that detection or tracking of relevant, important features and landmarks can be achieved \cite{santos2021multi,wei2021multi}. Centroid Voronoi tessellation~\cite{du1999centroidal} is an extensively studied technique for sensor coverage. Distributed controllers seeking this objective are proven to be scalable since they only consider the local neighborhood. One of the pioneering works in this domain is due to~\cite{cortes2004coverage}, which presents a Lloyd-algorithm-based distributed controller that drives the robotic sensors toward the centroids of their Voronoi partitions. Extensions of this controller have been presented in different settings such as time-varying intensity~\cite{cortes2010coverage} and limited communication range~\cite{laventall2009coverage}. Recursive partitioning of Voronoi partitions based on discovering an obstacle-free area in an unknown environment by a group of mobile robots is studied in \cite{dutta2020multi}. However, these algorithms assume that the robots are homogeneous with similar characteristics. Several studies in the literature formulate control laws for the coverage of heterogeneous robots. Pimenta et al.~\cite{pimenta2008sensing} consider heterogeneity in the sensing radius of the robots. Marier et al. ~\cite{Marier2013} have utilized sensor health as a weight metric -- a higher weight corresponds to a larger area coverage responsibility. Pierson et al.~\cite{pierson2016adaptivetrustweighting} learn the difference in robot sensing performance and adapt the weights online to assign larger areas to robots with better sensor health. 
A variation of this work is presented in \cite{pierson2015adapting} to adapt weights according to relative actuation errors. 

It is vital to consider the energy limitations of the robots relative to their neighborhood to increase the lifespan of the robot group \cite{setter2016energy, wei2020energy}. 
To deal with energy-related heterogeneity, Kwok and Martinez \cite{kwok2007energy} proposed an energy-aware variant of Lloyd's algorithm that uses a power diagram to balance the robot network's energy depletion and assign areas of higher importance to robots with higher energy levels while restricting the movements of robots with less energy. 
In \cite{duca2020multi}, the authors extended this idea to scenarios where the importance of the environmental distribution is time-varying. Morraef and Rodrigues~\cite{moarref2014optimal} ensure the maintenance of robots' energy levels by considering coverage as an optimal control problem while using the velocity of the robots as a weight metric. 
This controller assigns larger areas to the robots with higher speeds to ensure faster response.
Derenick et al.~\cite{derenick2011energy} studied the problem of maintaining persistent coverage by deriving control laws that allow robots with depleted batteries to reach corresponding access points (docking) for recharging stations. 
Authors in \cite{kim2022} explore heterogeneity in robot speed capability and propose a velocity-adaptive control law that partitions regions among robots based on the maximum velocities of robots to facilitate rapid response to time-sensitive applications. 
Limiting the velocities is a good approach in cases where robots need to optimize their velocity based on environmental constraints or to speed up the mission efficiency. 
However, adapting the partition weights based on velocities could be ignorant to the dynamics of energy consumption, which is influenced by various factors, such as computation, sensor power and frequency of operation, mechanical payload, communication range, altitude of operation, resource supply, etc.  \cite{lin2018energy,latif2021energy,coffey2023}. 

\section{Methodology}
Formally, the problem here is to devise a position controller that would balance the coverage area in a heterogeneous group of robots based on the robots' energy characteristics and optimize the locational (or geometric configuration) cost of the coverage. 
For instance, robots with higher energy depletion rates should be given less area to cover or monitor because they will quickly drain their batteries. 


\noindent
\textbf{Assumptions:} We assume that the robots do not have an accurate model of time-varying energy depletion rates, but they can instantaneously measure their current energy levels (ubiquitously available in robots in terms of battery voltage or percentage). 
We further assume that the robots form a connected network to share data.

\noindent
\textbf{Environment Partition. }
Let $R=\{1,2,\cdots n\} \in \mathbb{R}^n$ denote the set of $n$ robots, forming a connected graph $\mathcal{G} = (R,E)$ defined by their edges $E = \{(i,j) \}$ that share a connection link and $\mathcal{N}_i = \{j \in R | (i,j) \in E \}$ is the neighbor set of robot $i$ in the graph $\mathcal{G}$. The cardinality of the set $|\mathcal{N}_i| \in \mathbb{N}$ represents the number of neighbors for robot $i$.
The goal of the robots is to collectively cover a convex environment $\mathcal{Q} \subset \mathbb{R}^2$.  Let $p_i$ denote the position of robot $i$ and $q$ represent a point in $\mathcal{Q}$. Prior studies on partitioning robots for coverage have popularly used Voronoi partitioning~\cite{voronoi1908news}. Let $\mathcal{V}_i$ denote the partition for robot $i$, which can be calculated as the following. 
\begin{equation}
\mathcal{V}_i=\left\{q \in \mathcal{Q} \mid\left\|q-p_i\right\| \leq\left\|q-p_j\right\|, \forall j \neq i, j\in R\right\}
\end{equation}
A locational cost of the current partitions can be obtained as
\begin{equation}
H_{V}(p_1, ..., p_n) = \sum_{i=1}^n \int_{V_i} \frac{1}{2} ||q - p_i||^2 \phi(q) dq ,
\label{eq: tarditionalLLoydCost}
\end{equation}
where $\phi(q)$ is a density function $\mathcal{Q} \to R>0$ to describes the
importance of a given point $q$ \cite{cortes2004coverage}. In the coverage control literature, it is typical to operate under the assumption that the density function is known $\phi$ \cite{gosrich2022} (while some works try to learn the density before (or simultaneously) achieving coverage \cite{santos2021multi,munir2023exploration}). The traditional Lloyd-based coverage control algorithm minimizes the locational optimization cost in Eq.~\eqref{eq: tarditionalLLoydCost} by moving the robots toward the centroid of their Voronoi regions.
To incorporate heterogeneity in the robots, we will follow a weighted Voronoi partitioning (sometimes called power diagram) as mentioned in \cite{kwok2007energy,pierson2016adaptivetrustweighting}. The weighted partition $\mathcal{W}_i$ for a robot $i$ with a weight $w_i$ can be calculated as 
\begin{equation}
\mathcal{W}_i=\left\{q \in \mathcal{Q} \mid\left\|q-p_i\right\|^2-w_i \leq\left\|q-p_j\right\|^2-w_j,  \forall j \neq i\right\} ,
\label{eq: powerDiagram}
\end{equation}
The locational cost function for the power diagram will be 
\begin{equation}
H_{\mathcal{W}}(\textbf{p,w}) = \sum_{i=1}^n \int_{\mathcal{W}_i} \frac {1}{2}(||q - p_i||^2 - w_i) \phi(q) dq ,
\label{eq: weightedCost}
\end{equation}
where \textbf{p} = $\{p_1, p_2,.., p_n\}$ and \textbf{w} = $\{w_1, w_2,.., w_n\}$ are the vectors with all robot positions and weights, respectively.


It has been proven in the literature \cite{cortes2010coverage,pierson2015adapting,santos2018coverage} that the robots (or mobile sensors) should converge their positions at the centroid of their respective weighted Voronoi region ($C_{\mathcal{W}_i}$) to obtain the minimum locational cost in Eq.~\eqref{eq: weightedCost}, i.e., if the robots follow the gradient (partial derivative) of the cost function with respect to their positions, then they would eventually reach an equilibrium point where $C_{\mathcal{W}_i} = \argmin_{p_i} H_{\mathcal{W}_i}(p_i)$.
Therefore, the velocity controller for the robots, according to their weighted partitions, is given by 
\begin{equation}
\dot{p}_i = - k_p\left(C_{\mathcal{W}_i}-p_i\right) , \\
\label{eq: position controller}
\end{equation}
\begin{equation}
\text{where } C_{\mathcal{W}_i} = \frac{1}{M_{\mathcal{W}_i}} \int_{\mathcal{W}_i} q \phi(q) dq  \text{ and } M_{\mathcal{W}_i}  = \int_{\mathcal{W}_i}  \phi(q) dq .
\label{eq:centroid-mass}
\end{equation}

Here, $k_p \geq 0$ is the controller gain, which impacts the convergence rate. 
This control law places the robots in optimal locations for sensor coverage of the environment as long as the weights of the robots are adapted to satisfy the objectives of sensing \cite{pierson2016adaptivetrustweighting,santos2018coverage}, performance \cite{pierson2015adapting}, or energy heterogeneity \cite{kwok2007energy}. 
Our study focuses on adapting the weights $w_i$ based on the energy depletion rate, where the robots with better energy capability get assigned to larger partitions. 


\subsection{Energy Consumption Model}
\label{sec:energymodel}

Let $E_i(t)$ represent the current energy level of a robot $i$ at an instant $t$. 
According to the standard energy consumption model of mobile robots \cite{parasuraman2014model,mei2004energy}, three key parameters govern a robot's energy expenditure: current energy level (or the initial reserve), temporal energy cost (time-dependent costs such as the energy consumed by robot's computers and sensors even when the robot does not move), and spatial energy cost (mobility-dependent costs which characterize how energy depletes when robot moves with some velocity). Formally, we define the discretized energy consumption model as
\begin{align}
E_i(t) &= E_i(t-1) - \dot{E}_i(t)\\ 
\dot{E}_i(t) &=  - \alpha(t) - \beta(t) . |v_i(t)| 
\label{eq.ener_time}
\end{align}
Here, $\alpha(t)$ and $\beta(t)$ are the temporal and spatial energy depletion coefficients (potentially time-varying), respectively. The depletion coefficients need not be the same for all the robots. Assuming $\delta t$ is the time-lapse between the iterations and $\delta x$ is the displacement of the robot between the iterations, the velocity of the robot $v_i(t)$ is given by $\frac{\delta x}{\delta t}$. 
The energy depletion rate $\dot{E}_i(t)$ would vary depending on the heterogeneity in the system. 
Here, the values of $\alpha(t)$ and $\beta(t)$ together determine the energy depletion rate. Typically, they would not vary with time for a given robot unless it actively changes the behavior of how it consumes energy (e.g., hovering at a higher altitude, activating a new sensor payload, etc.). Therefore, with slight abuse of notation, we refer to $\alpha = \alpha(t), \beta = \beta(t)$.
For example, the $\alpha(t)$ and $\beta(t)$ of a drone will be multiple times higher than that of a ground vehicle, thus depleting energy much faster. Similarly, a ground robot with a 3D LIDAR and a camera will likely have higher $\alpha(t)$ than a ground robot with only a camera. 
In some robots, the mobility factor $\beta(t)$ is more significant than the temporal factor $\alpha(t)$. 
It is worth noting that the depletion rate has been largely ignored in the multi-robot literature, where planning is done based on the robot's velocity or path length (as a proxy for energy).

\subsection{Distributed Weight Adaptation Law}
We consider heterogeneity in energy consumption rates for adaptively estimating robot weights in partitioning the environment into $n$ Voronoi cells.
Accordingly, we propose a novel distributed energy-aware weight adaptation controller (EAC) that takes into account both the initial energy level and depletion rate of the robot.
\begin{equation}
   \dot{w}_i^{EAC} = -\frac{k_w}{M_{\mathcal{W}_i}}  \sum_{j \in \mathcal{N}_i}   
   \left( \frac{w_i}{w_j} - \frac{E_i^{init}}{E_j^{init}} . \frac{\dot{E}_j(t)}{\dot{E}_i(t)} \right)  
\label{eq: weightAdaptation_EAC}
\end{equation}

Here, $k_w$ is a positive gain constant depending on the environment size, $M_{\mathcal{W}_i}$ is the mass of the weighted Voronoi cell of robot $i$ (Eq.~\eqref{eq:centroid-mass}), and $E_i^{init}$ is the initial energy level of a robot $i$.
The depletion rate $\dot{E}_i$ of a robot at any instance will be a positive value as long as the robot remains powered, i.e., $\dot{E}_i(t) > 0$ as the $\alpha >0, \beta \geq 0$  (there will always be some energy consumed by the robot's computing unit even in idle state). We also assume there is no energy re-generation (back energy flow) during robot operation. The weight adaptation of the controller $\dot{w}_i^{EAC}$ is designed to balance the ratio of the weights with the neighbors by considering the direct ratio between the energy capacity of the robot and its neighbor and the inverse ratio between the energy depletion rate of the robot and its neighbor. Over time, the robot with a higher energy capacity will be assigned more weight, thereby more coverage area in the partitioning. Similarly, the robot with a higher energy consumption rate will get less weight to cover a smaller area.
The proposed distributed approach follows the below procedure. First, we find the weighted Voronoi partitions $\mathcal{W}_i$ based on the initial weights that are initialized to 1 for all robots $i \in R$. The energy depletion rates $\dot{E}_i(t)$ are calculated online once the robots start moving following Eq.~\eqref{eq.ener_time} and the corresponding change in Voronoi cell weights are calculated using Eq. \eqref{eq: weightAdaptation_EAC}. Next, for each such partition $\mathcal{W}_i$, its centroid $C_{\mathcal{W}_i}$ is calculated. The position control law in Eq.~\eqref{eq: position controller} is applied to find the velocities the robots should follow from their current locations $p_i$ to the centroid $C_{\mathcal{W}_i}$. The algorithm terminates if all the robots reach within $\epsilon$ distance from the centroid or if any of the robot's remaining energy at time $t$ is less than a threshold to indicate that the robot might not have enough battery power left for further operation.

Algorith ~\ref{algo:eac} provides a summary of the proposed energy-aware controller algorithm distributed in each robot.

\begin{algorithm}[t]
\KwIn{
$n$ robots, their positions, and $E_{init}$ (initial energy level or battery capacities).\\
$\epsilon, \delta:$ two small positive constants.\\
 }
 \KwOut{
    Energy-aware weighted region partitions.
 }
 \While{not converged}{
  \For{Each robot $r$} 
    {
    Find weighted Voronoi partition $W_r$ (Eq.~\eqref{eq: powerDiagram}). \\
    Find the centroid $C_r$ of $W_r$.\\
    Apply the position controller $\dot{p}_r$ given in Eq \eqref{eq: position controller}.\\    
    Get information on neighbors' energy and energy depletion rate. \\
    \If{ $\exists\dot{E}_i(t) - \dot{E}_i(t-1) > 0.2$}
    {
        $E_{init} = E_{current}$
    }
    Apply the energy-aware weight adaptation controller $\dot{w}_r$ from Eq. \eqref{eq: weightAdaptation_EAC} 
    
    Update $w_r$ and $p_r$
    
    }
    \If{$(\exists E_i(t) < \delta) || (C_i - p_i \leq \epsilon, \forall C_i, p_i)$}{
      convergence = True;\\
      return the energy-aware weighted region partitions.
   }
 }
\caption{Energy-Aware Coverage (EAC)}
\label{algo:eac}
\end{algorithm}





\begin{theorem}
Applying the distributed energy-aware coverage approach with the weight adaptation controller in Eq.~\eqref{eq: weightAdaptation_EAC} to robots following the energy dynamics Eq.~\eqref{eq.ener_time} will asymptotically reach a balanced ratio of the weights $w_i$ based on the energy depletion rate ratio, i.e., 
\begin{equation}
(w_i / w_j) \rightarrow (\dot{E_j} / \dot{E_i}) \quad \forall i,j    
\label{eq: weightController_Theorem1}
\end{equation} 
\label{thm: weightlaw}
\end{theorem}
\begin{proof}[Theorem 1]
To prove that the weight adaptation law $\dot{w}_i$ applied to all robots will asymptotically lead the multi-robot system to an equilibrium state defined in Eq.~\eqref{eq: weightController_Theorem1}, we introduce a Lyapunov candidate function $\mathscr{V}$ and show that the derivative of the Lyapunov function is negative semidefinite. Consider
\vspace{-2mm}
\[\begin{aligned}
\centering
\mathscr{V} = \sum_{i=1}^n \frac{1}{2} \|w_i \dot{E_i} \|^2 .
\nonumber    
\end{aligned}
\]
\vspace{-2mm}
The time-derivative of $\mathscr{V}$ with respect to the weights $\textbf{w}$ is
\[
\begin{aligned}
\dot{\mathscr{V}} & = \frac{\partial \mathscr{V}}{\partial \textbf{w}} = 
\sum_{i=1}^n\left(w_i \dot{E_i}\right)^T \dot{w}_i 
& = - \sum_{i=1}^n \frac{k_w \left(w_i \dot{E}_i\right)^T}{{M_{\mathcal{W}_i}} } \sum_{j \in \mathcal{N}_i}   
   \left( \frac{w_i}{w_j} - \frac{E_i^{init}}{E_j^{init}} . \frac{\dot{E}_j}{\dot{E}_i} \right) 
\end{aligned}         
\]
For tractability, let's assume that $E^{init}$ is the same for all robots and, without loss of generality, assume $\dot{E}_i(t) = \dot{E}_i $ (a constant). This simplifies the above derivation to
\begin{equation}
\begin{aligned}
\dot{\mathscr{V}} & = - \sum_{i=1}^n \frac{k_w \left(w_i \dot{E}_i\right)^T}{{M_{\mathcal{W}_i}} } \sum_{j \in \mathcal{N}_i}   
   \left( \frac{w_i}{w_j} - \frac{\dot{E}_j}{\dot{E}_i} \right) 
 & =  - \sum_{i=1}^n k_w \frac{1}{{M_{\mathcal{W}_i}} } \sum_{j \in \mathcal{N}_i}  \frac{w_i}{w_j} \left( w_i \dot{E}_i - w_j  \dot{E}_j \right) 
\end{aligned}  
\nonumber
\label{eq: vdot}
\end{equation}
Re-writing this expression in matrix form gives us the flexibility to analyze the resulting product of this multivariate expression. Let's define
\[
\begin{aligned}
\Tilde{w}_e =  \begin{bmatrix} w_1 \dot{E}_1  \\ \vdots \\ w_n \dot{E}_n    \end{bmatrix}
\; , &  \;
M^{-1} = \begin{bmatrix} \frac{1}{M_{\mathcal{W}_1}} &0 & 0 \\  0 & \ddots &0 \\ 0 &0 &\frac{1}{M_{\mathcal{W}_n}} \end{bmatrix} .
\end{aligned}
\]
After manipulations, we obtain
\[
\begin{aligned}
\dot{\mathscr{V}} & = - k_w \Tilde{w}_e^T M^{-1} \mathbb{L} \Tilde{w}_e \;  \leq 0 .
\end{aligned}  
\]
Here, $\mathbb{L}$ is the Laplacian matrix of graph $\mathcal{G}$. Given that $M^{-1}$ is a positive diagonal matrix and the fact that a graph Laplacian $\mathbb{L}$ is positive semi-definite, we can see that the $M^{-1}\mathbb{L}$ is positive semi-definite.
This leads us to the conclusion that $\dot{\mathscr{V}}$ is negative semidefinite.
According to La Salle's Invariance principle \cite{la1976stability}, the largest invariance set is at $\dot{\mathscr{V}} = 0$. This equilibrium state can be reached when $\Tilde{w}_e$ falls in the null space of $\mathbb{L}$. We know that for a connected graph, the smallest eigenvalue of its Laplacian is always zero with an eigenvector $\Vec{1}$. Therefore, we obtain $\mathscr{V} = 0$ if all the entries in the vector $\Tilde{w}_e$ are identical, i.e., $ (w_i \dot{E}_i = w_j \dot{E}_j ) \, \forall i,j$. When we achieve this, the ratio of the weights becomes $(w_i / w_j) = (\dot{E_j} / \dot{E_i})$, concluding the proof of Eq.~\eqref{eq: weightController_Theorem1}.
The results can now be extended to time-varying $\dot{E}(t)$, where the weights would periodically adjust their ratios as per Eq.~\eqref{eq: weightController_Theorem1} as long as $\dot{E}$ remains constant within the convergence process.  
\end{proof}

\begin{corollary}
We can apply the results of Theorem~\ref{thm: weightlaw} to balance the ratio of the combination of initial energy and the energy depletion rates.
\begin{equation}
(w_i / w_j) \rightarrow ({E^{init}_i} \dot{E}_j) / ({E^{init}_j} \dot{E}_i) \quad \forall i,j \label{eq: weightcombo_Theorem1}
\end{equation} 
\label{thm: weightcombo}
\end{corollary}
\vspace{-10mm}
\begin{proof}
The combination of reserve energy $E_i$ and depletion rate $\dot{E}_i$ can be generalized to a new weighted energy depletion rate $\dot{\mathbb{E}}_i = k_d \dot{E}_i$, which can replace the depletion rate in Theorem~\ref{thm: weightlaw} to obtain $(w_i / w_j) = ({\dot{\mathbb{E}}_j}/{\dot{\mathbb{E}}_i})$.
\end{proof}

\begin{theorem}
Using the position controller Eq. \eqref{eq: position controller} and the weight controller Eq.~\eqref{eq: weightAdaptation_EAC} for all robots in a distributed manner, the multi-robot system will asymptotically converge to the stable equilibrium towards the minimum of the locational cost in Eq.~\eqref{eq: weightedCost}. i.e., 
\begin{align}
\mathbf|{p}_i - \mathbf{C}_{W_i}| \rightarrow 0\ \forall i \in n.
\label{eq: positonController_Theorem1}
\end{align} 
\label{thm: positioncontroller}
\end{theorem}
\vspace{-8mm}
\begin{proof}
Let us treat the cost function $H_{\mathcal{W}}$ as a Lyapunov candidate function to demonstrate that the controller drives the multi-robot system to optimal coverage positions. The time derivative of $H_{\mathcal{W}}$ is
\begin{equation}
\dot{H}_{\mathcal{W}} =\sum_{i=1}^n \int_{\mathcal{W}_i}\left(q-p_i\right)^T \phi(q) d q \dot{p}_i+\sum_{i=1}^n \int_{\mathcal{W}_i} \frac{1}{2} \phi(q) d q \dot{w}_i
\nonumber 
\end{equation}
Splitting the above equation into two parts for traceability $\dot{H}_{\mathcal{W}} = \dot{H}_1 + \dot{H}_2$.
By applying the position controller $\dot{p}_i$ in Eq.~\eqref{eq: position controller},
\[\begin{aligned}
\dot{H}_1 & = \sum_{i=1}^n \int_{\mathcal{W}_i} \left(q-p_i\right)^T \phi(q) d q \ [k_p(C_{\mathcal{W}_i}-p_i)]  
&= \sum_{i=1}^n -k_p M_{\mathcal{W}_i} [(C_{\mathcal{W}_i}-p_i)]^2 \leq 0
\end{aligned} \]

Applying the weight adaptation law in Eq.~\eqref{eq: weightAdaptation_EAC} 
\[
\begin{aligned}
\dot{H}_2 & = - \sum_{i=1}^n \frac{1}{2} M_{\mathcal{W}_i} \frac{k_w}{M_{\mathcal{W}_i}}  \sum_{j \in \mathcal{N}_i}   
   \left( \frac{w_i}{w_j} - (\frac{E_i^{init}}{E_j^{init}} . \frac{\dot{E}_j(t)}{\dot{E}_i(t)}) \right)
\end{aligned}                  
 \]
Applying the result of Theorem~\ref{thm: weightlaw} and Corollary~\ref{thm: weightcombo}, we obtain
\[
\begin{aligned}
\dot{H}_2 & \approx - \sum_{i=1}^n \frac{k_w}{2} \sum_{j \in \mathcal{N}_i} \left( \frac{w_i}{w_j} - \frac{w_i}{w_j} \right)  \approx 0 .
\end{aligned}                  
 \]
 
Therefore, $\dot{H}_{\mathcal{W}} \leq 0$, proving the asymptotic convergence of Eq.~\eqref{eq: positonController_Theorem1}.
$\dot{H}_{\mathcal{W}} = 0$ only when the velocity $\dot{p}_i  = 0$ for all robots $i \in n$. This can occur only when all the robots reach the centroid of their weighted Voronoi configuration, i.e., the robots converge to their centroids $p_i = C_{\mathcal{W}_i}$, which is the largest invariance set. This concludes the proof.
\end{proof}

Theorem~\ref{thm: weightlaw} implies the convergence of the relative ratio between the weights and the energy dynamics (rather than the weights being proportional to the energy, as in \cite{kwok2007energy}) in the system. This enables our objective of obtaining coverage area proportional to the energy characteristics, as depicted in Fig.~\ref{fig:illust}. 
Together with the centroid-seeking controller in Theorem~\ref{thm: positioncontroller}, the controller will achieve optimum coverage in the environment with the least cost in Eq.~\eqref{eq: weightedCost}. Moreover, the convergence in the weight adaptation depends on distributed coordination and happens much faster than the convergence of positions towards the weighted centroids \cite{parasuraman2019consensus}.
This allows EAC to work with time-varying energy consumption, making our controller realistic for different robot types and applications, as well as adaptable to heterogeneous robots with dynamic energy models. 

\section{Simulation Experiments and Results}

\begin{table}[t]
\centering
\caption{Description of various scenarios tested in the simulation for six robots. Highlighted in red boldface are the key differences in the settings. We also list the final weights obtained by the different approaches for energy-aware coverage partitioning. }
\label{table:exp_scenarios}
\resizebox{1\linewidth}{!}{
\begin{tabular}{|c|c|c|c|c|}
\hline
\multirow{2}{*}{\textbf{Scenario}} & \multirow{1}{*}{\textbf{Initial Energy Reserve}} &  \multicolumn{3}{c|}{Energy Depletion Rate $N=6$}  \\ 
\cline{3-5}
 & $E_i^{init}$ (\%) &  \textbf{Temporal} $\boldsymbol{\alpha_i}$ &  \textbf{Spatial} $\boldsymbol{\beta_i}$ 
& Total Depletion $\boldsymbol{\dot{E}_i(t)}$ \\

\hline

1 & \{100,100,100,100,100,100\} & \{1,1,1,1,\textbf{\textcolor{red}{5}},1\} & \{1,1,1,1,1,1\} 
& \{1.4,1.4,1.4,1.4,\textbf{\textcolor{red}{5.4}},1.4\}  
\\ \hline
2 & \{25,25,25,\textbf{\textcolor{red}{100}},25,25\} & \{1,1,1,\textbf{\textcolor{red}{4}},1,1\} & \{1,1,1,1,1,1\} 
& \{1.4,1.4,1.4,\textbf{\textcolor{red}{4.4}},1.4,1.4\} 
\\ \hline
\multirow{3}{*}{3} & @ t=0, $E^{init}_i$ = 100 \% & \{1,1,1,1,1,1\} & \{1,1,\textbf{\textcolor{red}{5,5}},1,1\} 
& \{1.4,1.4,\textbf{\textcolor{red}{3.0}},\textbf{\textcolor{red}{3.0}},1.4,1.4
\} \\
 & @ t=11 & \{1,1,1,1,1,1\} & \{\textbf{\textcolor{red}{5,5,1,1,5,5}}\} 
& \{ \textbf{\textcolor{red}{3.0,3.0,1.4,1.4,3.0,3.0}}\}  \\
 & @ t=22 & \{1,1,1,1,1,1\} & \{\textbf{\textcolor{red}{1,1,5,5,1,1}}\} 
& \{\textbf{\textcolor{red}{1.4,1.4,3.0,3.0,1.4,1.4}}
\}
\\ \hline
\end{tabular}
}
\vspace{2mm} \\
\resizebox{1\linewidth}{!}{
\begin{tabular}{|c|c|c|c|c|}
\hline
\multirow{2}{*}{\textbf{Scenario}} & \multirow{2}{*}{Time Instants} & \multicolumn{3}{c|}{\textbf{Final weights $w_i$ of all robots. (Initial weights = \{1,1,1,1,1,1\})}} \\ 
\cline{3-5}
 & &  \textbf{EAC (ours)} & \textbf{ATC \cite{pierson2016adaptivetrustweighting}} &  \textbf{PBC \cite{duca2020multi}}  \\

\hline

1 & Final & \{1.5,1.5,1.5,1.5,0.4,1.5\} & \{1,1.1,1.1,1.1,0.6,1.1\}  & \{-0.8,-0.9,-0.9,-0.9,-1.0,-0.9\} 
\\ \hline
2 & Final  & \{1,1,1,1.2,1,1\} & \{1,1.1,1,0.6,1,1\} &  \{-1,-1,-1,-0.9,-1,-1\}
\\ \hline
\multirow{3}{*}{3} & @ t=0 &  \{1.4,1.4,0.6,0.6,1.4,1.4\} & \{1.1,1.1,0.7,0.7,1.1,1.1\}  & \{-1,-1,-1,-1,-1,-1\} \\
 & @ t=11 & \{3.0,3.0,1.4,1.4,3.0,3.0\}  & \{1.1,1.1,0.7,0.7,1.1,1.1\}  & \{-1,-1,-1,-1,-1,-1\} \\
 & @ t=22 (Final) & \{1.4,1.4,0.6,0.6,1.4,1.4\} & \{1.1,1.1,0.7,0.7,1.1,1.1\}  & \{-1,-1,-1,-1,-1,-1\} \\ 
 \hline
\end{tabular}
}
\end{table}

\def\figwidth{0.23}
\begin{figure*}[t]
    \centering
    \begin{tabular}{ccccc}
        \textbf{\underline{Initial}} & \textbf{\underline{EAC}} & \textbf{\underline{ATC}} &
        \textbf{\underline{WMTC}} & \textbf{\underline{PBC}} \\
        \hspace{-0.05in}\includegraphics[width=\figwidth\columnwidth]{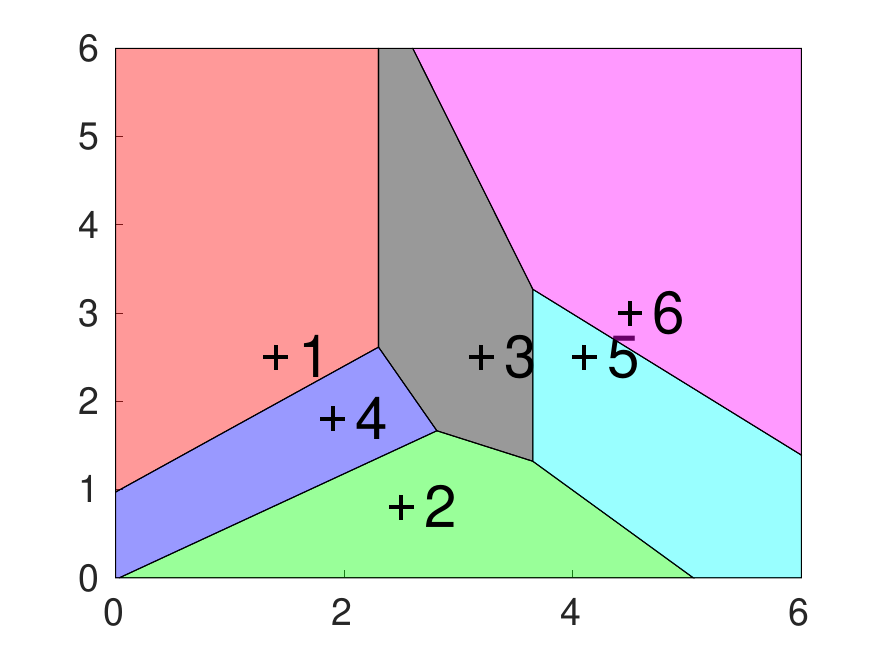} &
        \hspace{-0.2in}\includegraphics[width=\figwidth\columnwidth]{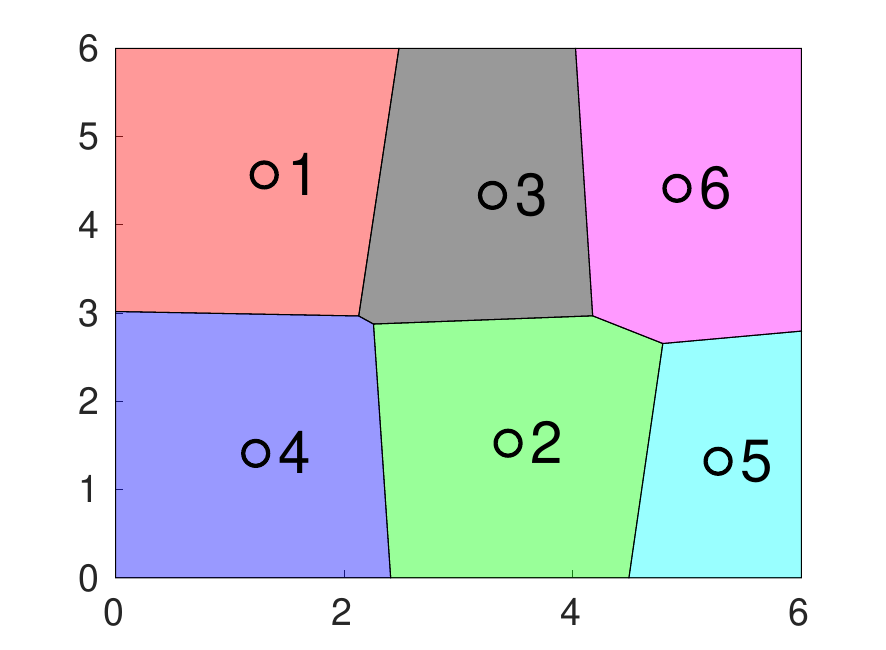} &
        \hspace{-0.2in}\includegraphics[width=\figwidth\columnwidth]{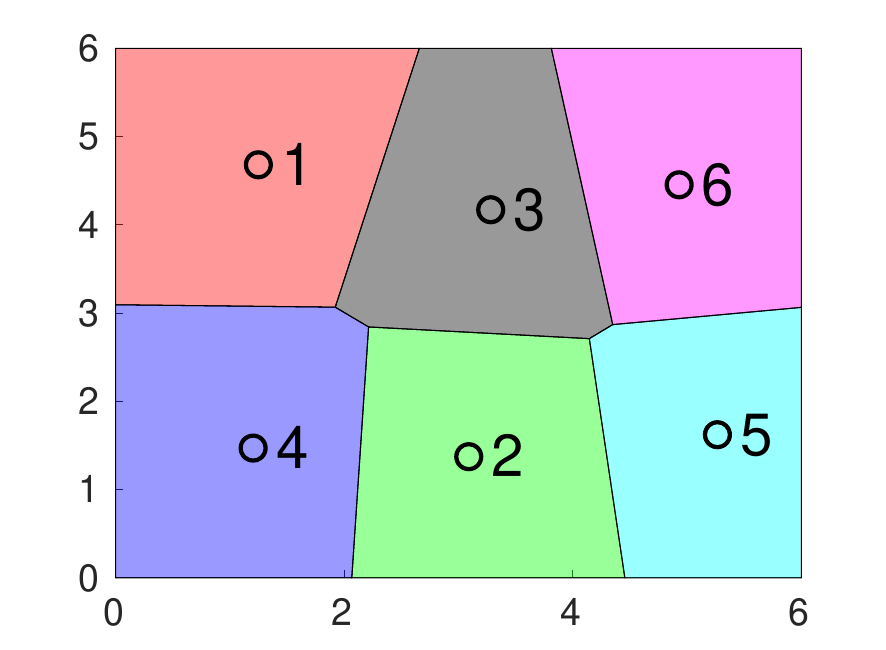 } &
        \hspace{-0.2in}\includegraphics[width=\figwidth\columnwidth]{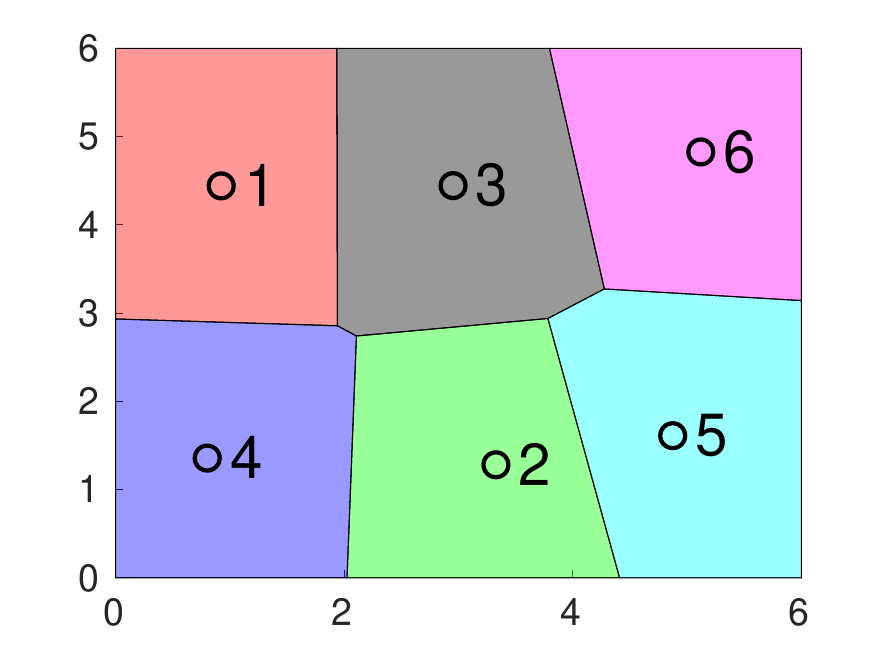} &
        \hspace{-0.2in}\includegraphics[width=\figwidth\columnwidth]{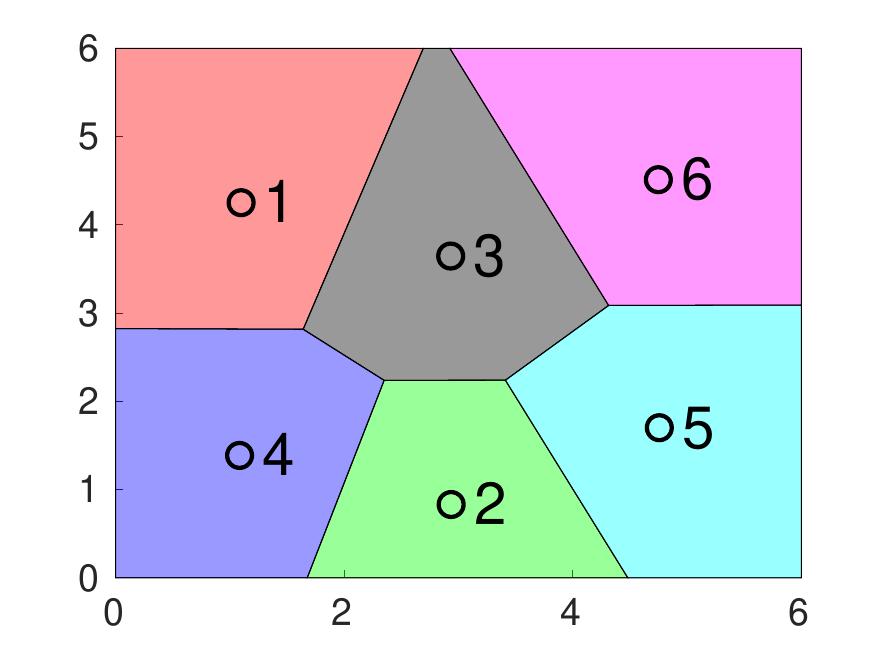} \\
        \hspace{-0.05in}\includegraphics[width=\figwidth\columnwidth]{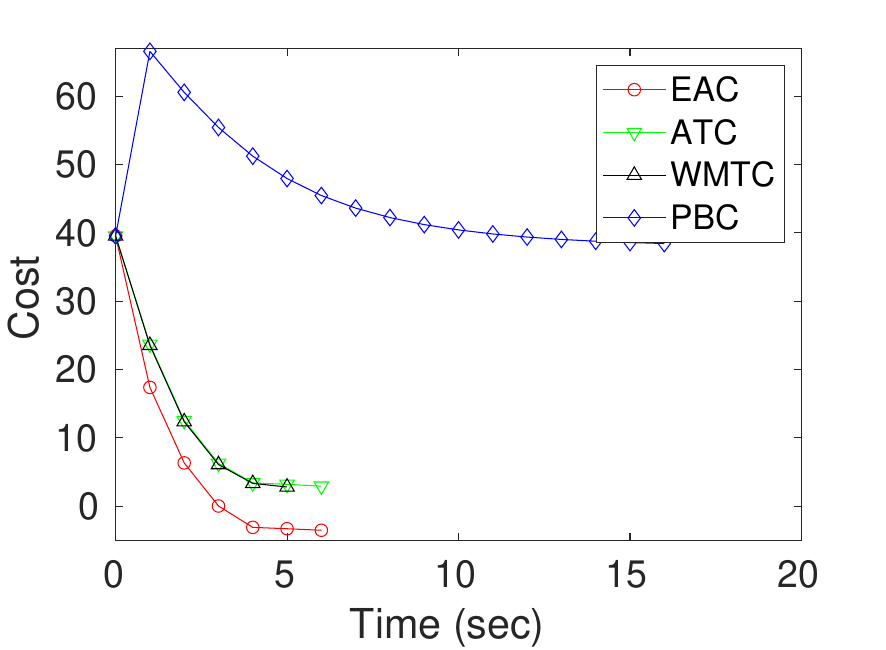} &        
        \hspace{-0.2in}\includegraphics[width=\figwidth\columnwidth]{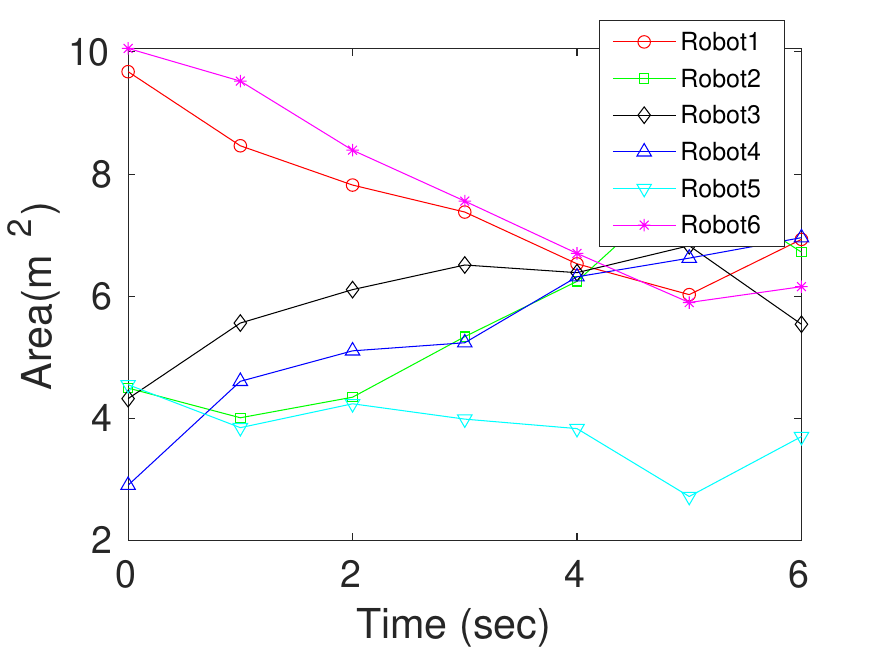} &
        \hspace{-0.2in}\includegraphics[width=\figwidth\columnwidth]{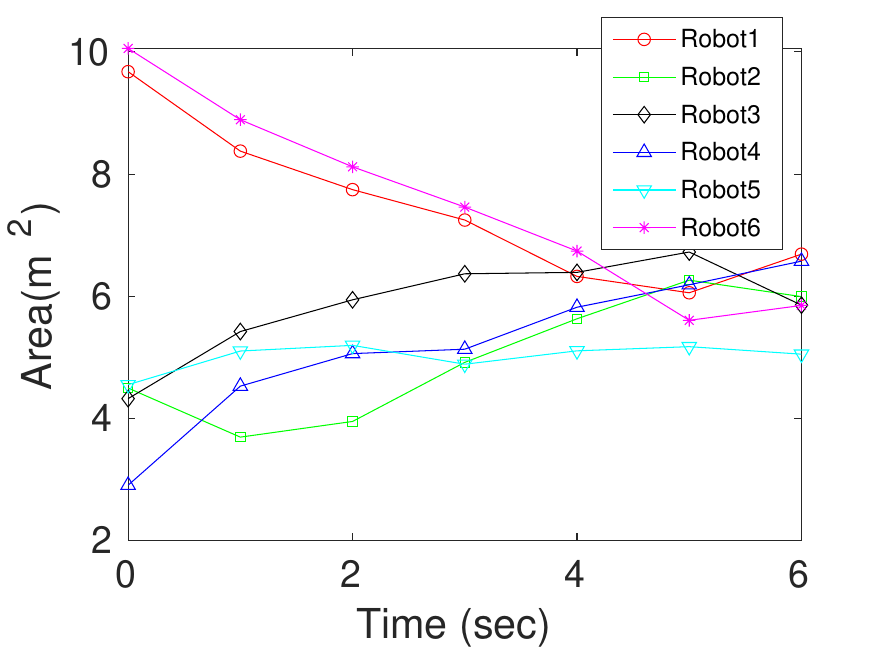} &
        \hspace{-0.2in}\includegraphics[width=\figwidth\columnwidth]{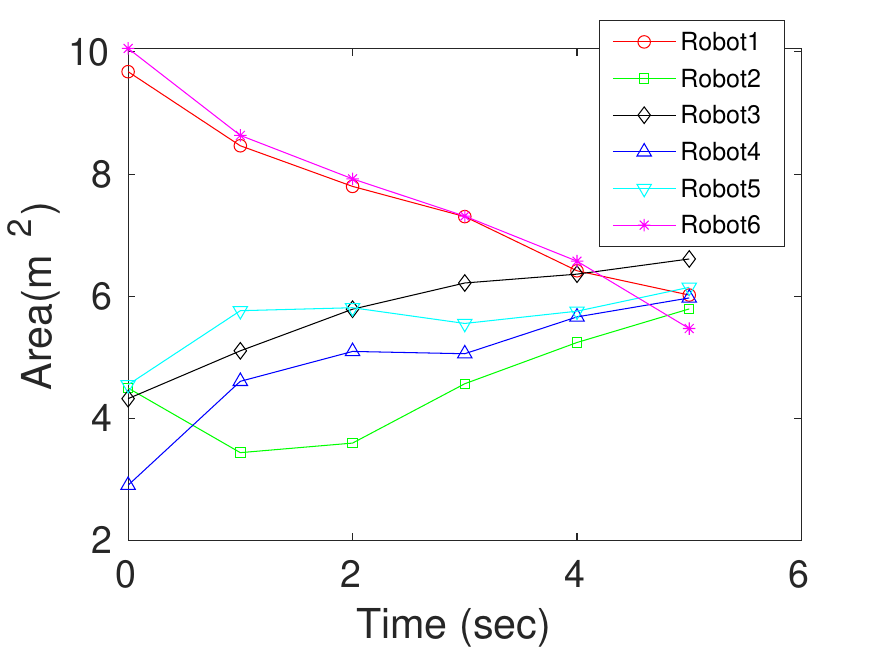} &
        \hspace{-0.2in}\includegraphics[width=\figwidth\columnwidth]{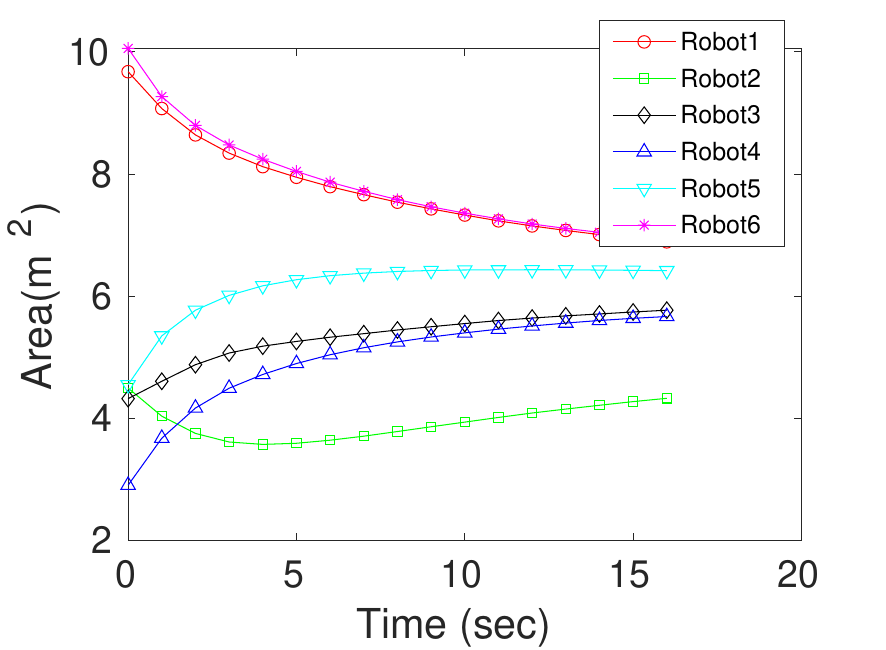}
    \end{tabular}
    \vspace{-4mm}
    \caption{Results of Scenario 1, where all robots have the same energy characteristics $E_i(0)=100, \alpha_i=1, \beta_i=1$, but robot $5$ has a higher temporal energy depletion rate $\alpha_5=5$. The top plots show the initial configurations and the final partitions calculated by all the approaches. The bottom plots show the cost comparison (left-most) and the area convergence over time.}
    \label{fig:results_E2}
    \vspace{-4mm}
\end{figure*}

We compared our proposed EAC controller against three relevant coverage controllers from the literature, as listed below. 

\begin{itemize}
\item \textbf{WMTC} - As a baseline, we implement 
the weighted Voronoi (power diagram) in Eq.~\eqref{eq: powerDiagram} with constant and equal weights for all robots \cite{cortes2010coverage}. It uses the move-to-centroid (MTC) position controller in Eq. \eqref{eq: position controller}.

\item \textbf{PBC} - We implement the control law from \cite{duca2020multi}, herein referred to as the Power Balance Controller (PBC). Although the authors focus on the time-varying density function $\phi(t)$ (which affects the centroid $C_{\mathcal{W}_i}$ calculation), their energy-aware control law is useful for comparison.
    \begin{align*}
\dot{p}_{i} &= -k(E_{i}) 
\begin{cases}
    p_{i}-C_{\mathcal{W}_i}, & \text{if } \Vert p_{i}-C_{\mathcal{W}_i}\Vert \leq 1, \\
    \frac{p_{i}-C_{\mathcal{W}_i}}{\Vert p_{i}-C_{\mathcal{W}_i}\Vert}, & \text{if } \Vert p_{i}-C_{\mathcal{W}_i}\Vert > 1.
\end{cases} 
\end{align*}
Here, the weights $w_i (t) = E_{max} - E_i (t)$ are adapted based on the robot's current energy budget, where $E_{max}$ is the maximum energy capacity. 

\item \textbf{ATC} - In ATC~\cite{pierson2016adaptivetrustweighting}, the robots with higher quality sensors get assigned larger areas to cover. We use ATC as a baseline as it adapts weighting online similarly to us, albeit for sensor health. In our implementation of ATC,  
the trust is inversely proportional to the energy depletion rate to assign lower weights (area) to robots with higher energy use. The position controller is the same as in Eq. \eqref{eq: position controller}. 
\begin{align*}
    \dot{w_i} = \frac{\alpha k_w}{2M_{W_i}} \sum_{j \in \mathscr{N}_i} \left( (w_j - w_j) - (e_j - e_i) \right) , e_i = (\frac{K_e}{\dot{E}_i(t)})^2 .
\end{align*}
$K_e$ is a constant to scale the energy depletion rates based on the environment size.
\end{itemize}
\vspace{-8pt}
To perform comparisons against the baselines fairly, we have set the velocity of all robots to $0.4 {m}/{s}$ (with $dt=1$) for all the controllers, except PBC, which assigns velocity based on energy. 
Also, we set the initial weights of all controllers to 1, considering normalized energy with respect to maximum energy capability $E_{max}$. During the experiments, all controllers will change their weights based on the current energy levels as per their implementation. For instance, the PBC considers the instantaneous energy level differences to adjust the robots' weights.
We have implemented the controllers described above in MATLAB simulations with $n=6$ (but the algorithm is built on top of the WMTC, which is a decentralized controller that can scale to any number of robots in a distributed graph $\mathcal{G}$ as noted in \cite{cortes2004coverage}). 
The environment $\mathcal{Q}$ is a $6 \times 6$ square region. The initial weights were set as $w_i(0) = 1,  \forall i \in R$. 
We consider a uniform density function (i.e., $\phi(q) = 1 \, \forall q$) to demonstrate the effectiveness of the controller on the energy-aware characteristics without being influenced by the density function $\phi$. 

We rigorously test the controllers in various scenarios, as discussed below. 
Table~\ref{table:exp_scenarios} provides the details of these scenarios and the resulting (final) weights of different controllers. Note the $\dot{E}$ is constant for all robots in the first three scenarios.
A baseline scenario (Scenario 0, not discussed here due to its simplicity) is when the energy-related parameters such as $E^{init}_i, \alpha_i, \beta_i$ of all the robots are the same (or identical) and all controllers result in the same partitioning as the WMTC in the scenarios below as WMTC does not depend on energy characteristics.

\begin{figure}[t]
    \centering
    \begin{tabular}{cccc}
        \textbf{\underline{Energy Level $E_i$}} & \textbf{\underline{Weights $w_i$}} & 
        \textbf{\underline{Weight Convergence $w_i\dot{E}_i$}} & \textbf{\underline{Distance to Centroid}} \\
        \hspace{-0.1in}\includegraphics[width=0.25\linewidth]{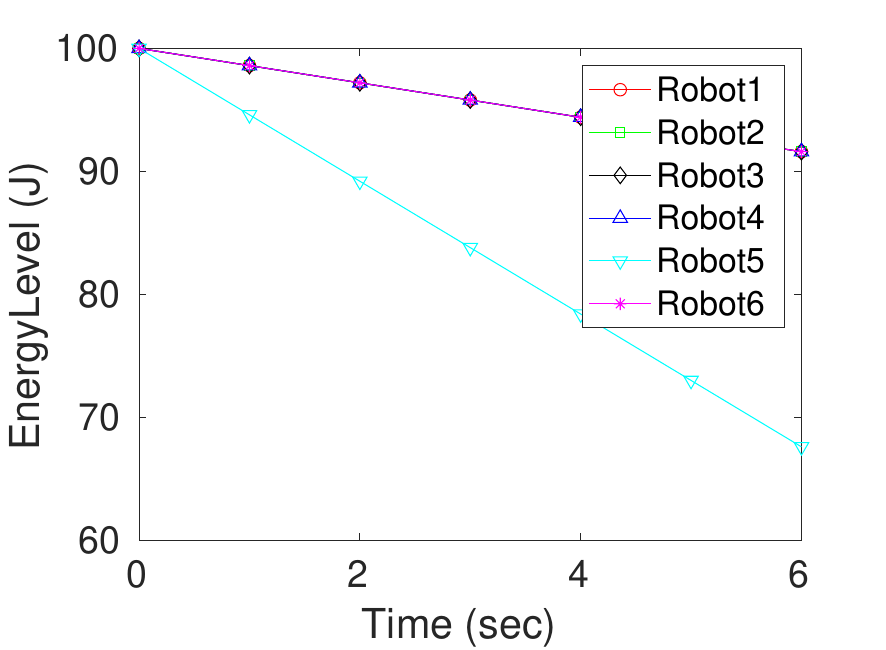} &
        \hspace{-0.1in}\includegraphics[width=0.25\linewidth]{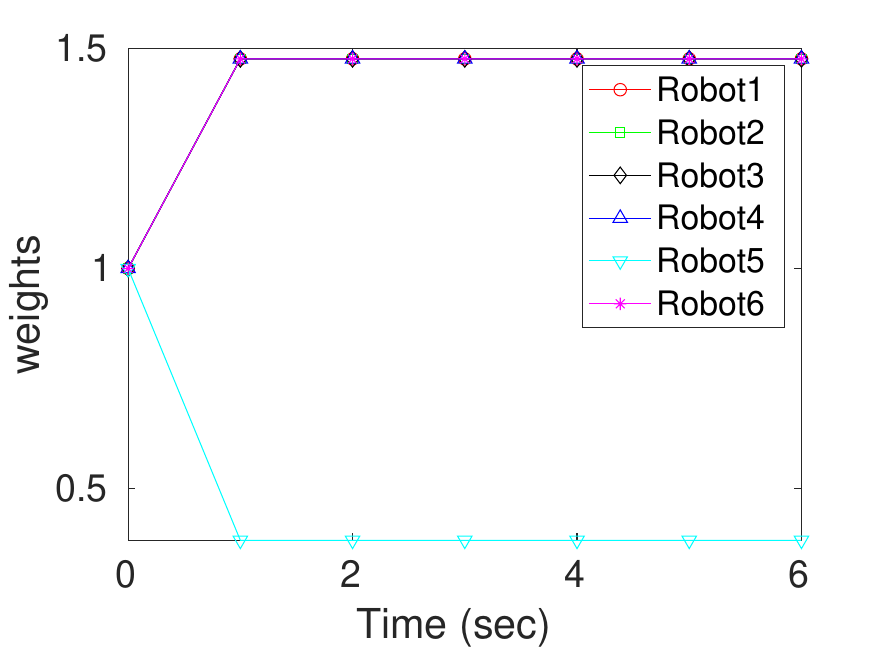} &
        \hspace{-0.1in}\includegraphics[width=0.25\linewidth]{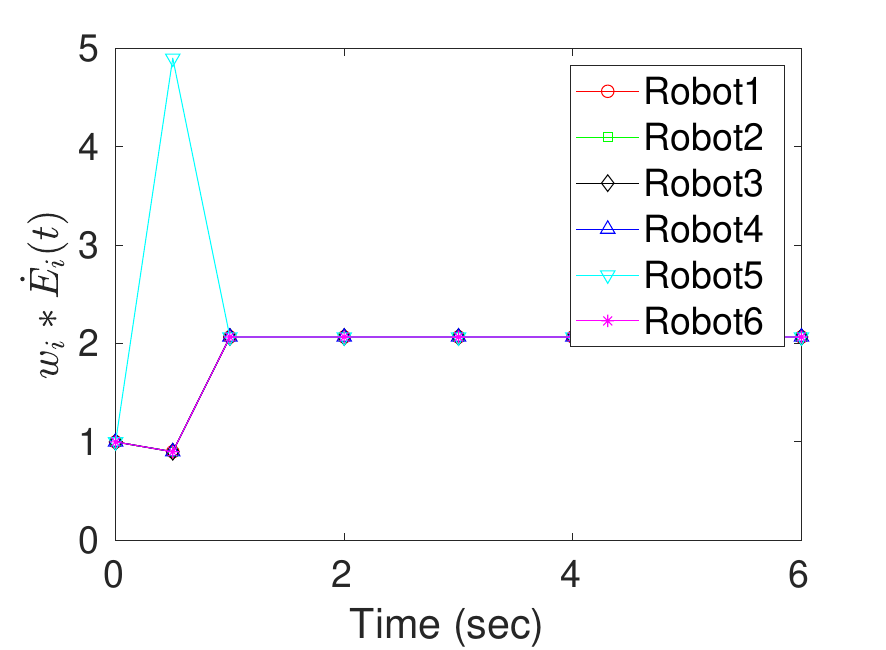} &
        \hspace{-0.1in}\includegraphics[width=0.25\linewidth]{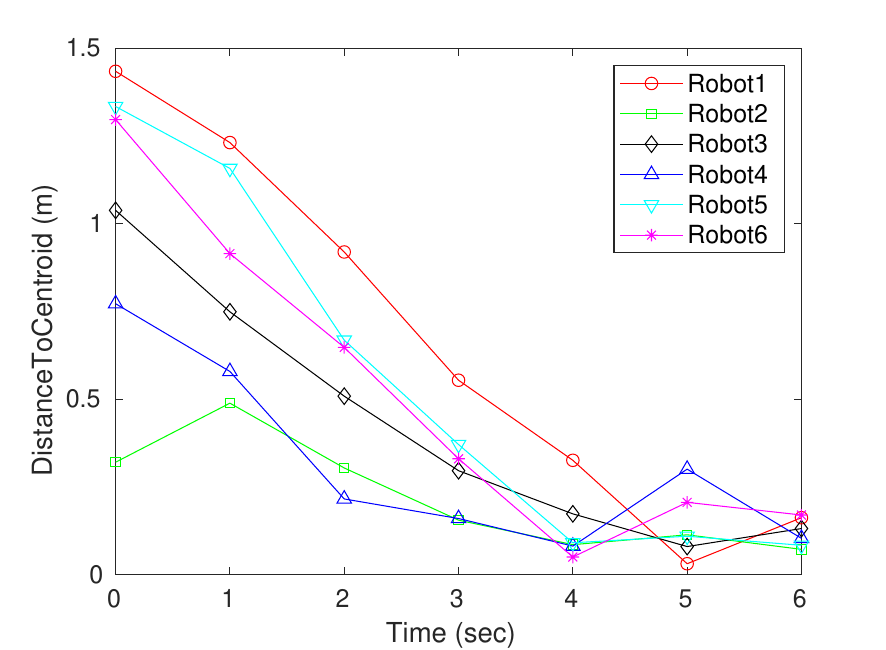} \\
    \end{tabular}
    \vspace{-4mm}
    \caption{Time-evolution of the robot energy levels $E_i$, weights $w_i$, and the convergences of weight ratios $w_i\dot{E}_i$ and the robot distance from the centroid of their respective Voronoi cell in Scenario 1 with EAC Controller.}
    \label{fig:weight_convergence}
\end{figure}

\subsection{Scenario 1 - Different Temporal Energy Rates}
This scenario illustrates how the region would be re-partitioned if temporal energy depletion ($\alpha$) for one robot is greater than others. 
In this setting, the robots start with the same battery reserve, $\alpha=1$, and $\beta=1$. 
But, robot $5$ is parameterized with a higher $\alpha=5$ showing the total energy depletion rate ($\dot{E}_i(t)=5.4$) compared to all other robots ($\dot{E}_i(t)=1.4$), prompting the controller to conserve energy by limiting the robot's allocated area, and consequently, reducing its travel cost. Following EAC, robot $5$ is assigned a weight ($0.4$) that is almost one-fourth of the weight assigned to other robots ($1.5$). As shown in Fig. \ref{fig:results_E2}, robot $5$'s region shrinks compared to the WMTC partition. Although ATC could also shrink the region of this robot, the cost reduction achieved by our proposed EAC is substantial, reaching approximately -3.5, while the costs of ATC, WMTC, and PBC are 2.9, 2.8, and 38.4, respectively. 
Also, PBC required significantly higher iterations to converge because of its inherent velocity constraints. 
Thus, EAC outperformed the baselines. Fig.~\ref{fig:weight_convergence} shows various metrics over time. For instance, the $w_i\dot{E}_i$ values of all robots quickly converged to the same numbers, supporting Theorem~\ref{thm: weightlaw}. 
Also, as the robots ran the controller in Eq.~\eqref{eq: position controller}, the robots asymptotically converged to the centroids (Theorem~\ref{thm: positioncontroller}).


\begin{figure*}[t]
    \centering
    \begin{tabular}{ccccc}
        \textbf{\underline{Initial}} & \textbf{\underline{EAC}} & \textbf{\underline{ATC}} &
        \textbf{\underline{WMTC}} & \textbf{\underline{PBC}} \\
        \hspace{-0.1in}\includegraphics[width=\figwidth\columnwidth]{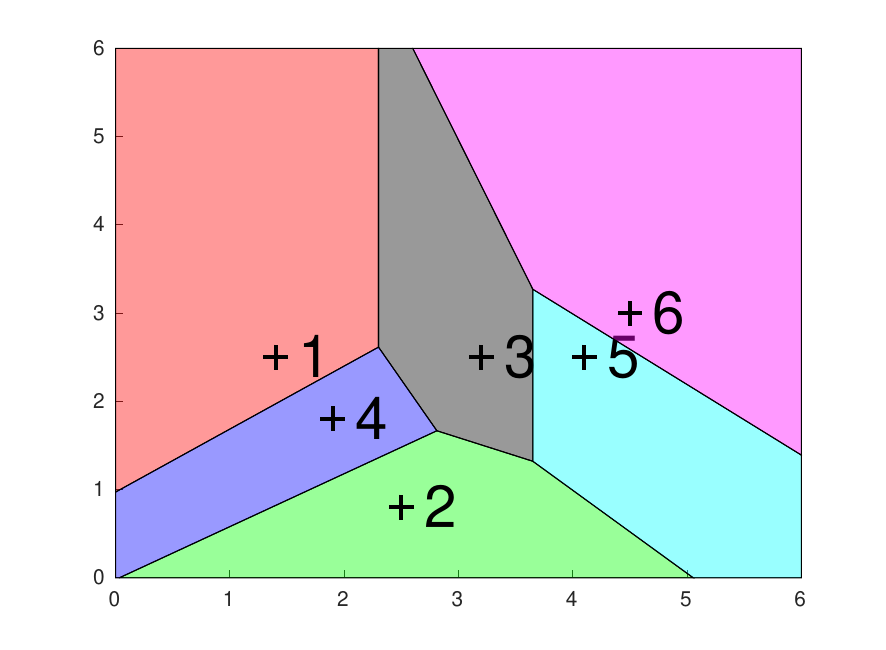} &
        \hspace{-0.2in}\includegraphics[width=\figwidth\columnwidth]{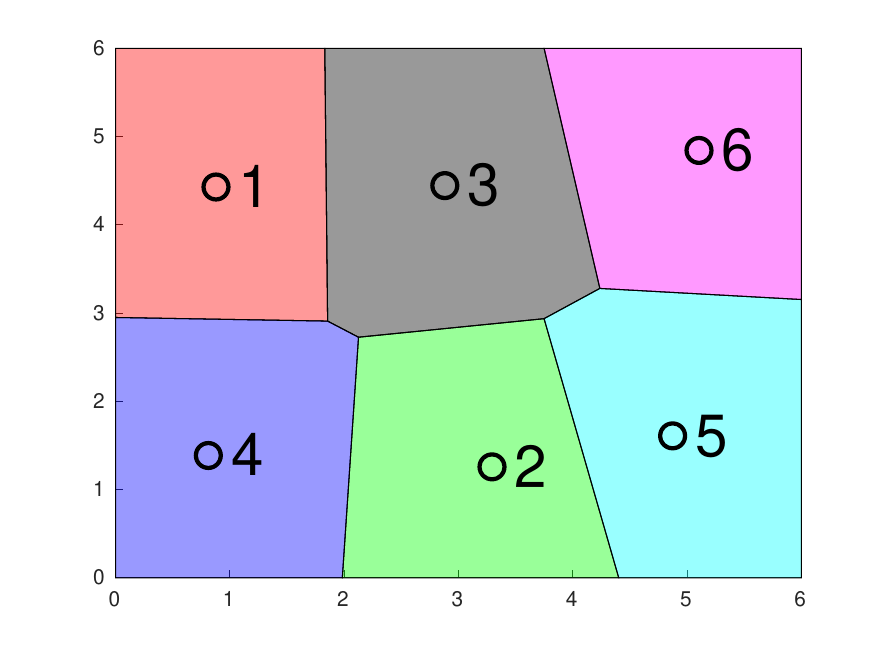} &
        \hspace{-0.2in}\includegraphics[width=\figwidth\columnwidth]{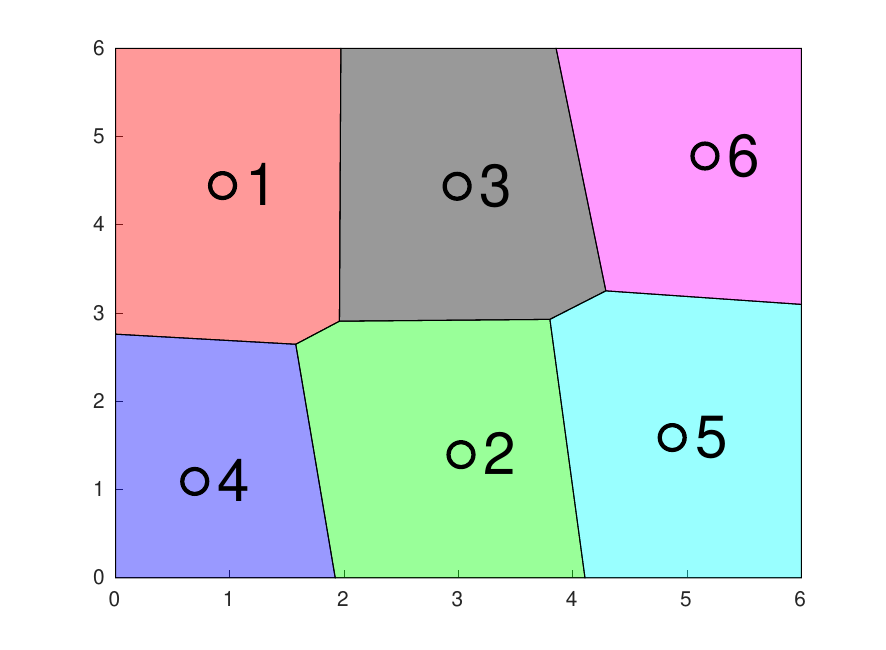 } &
        \hspace{-0.2in}\includegraphics[width=\figwidth\columnwidth]{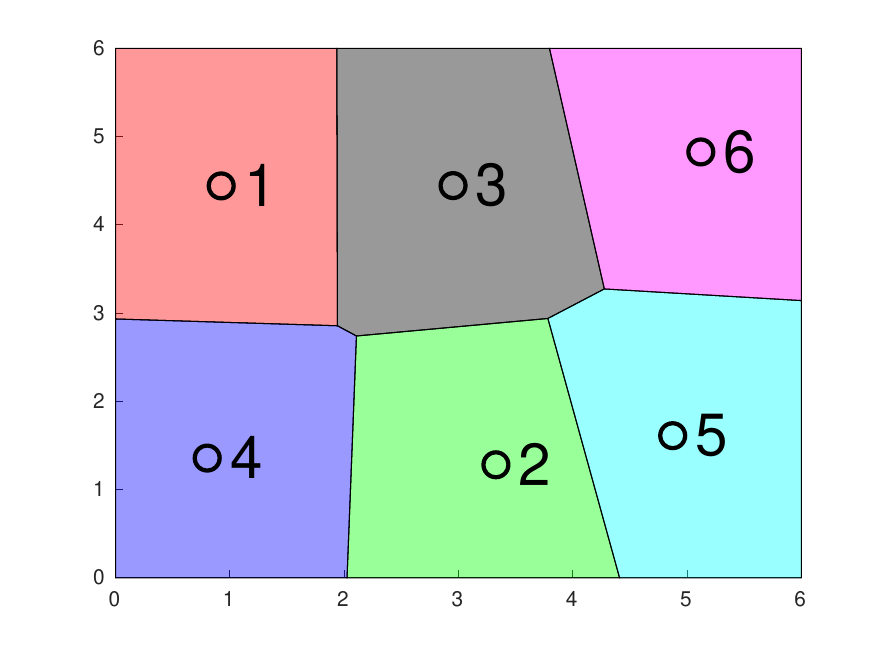} &
        \hspace{-0.2in}\includegraphics[width=\figwidth\columnwidth]{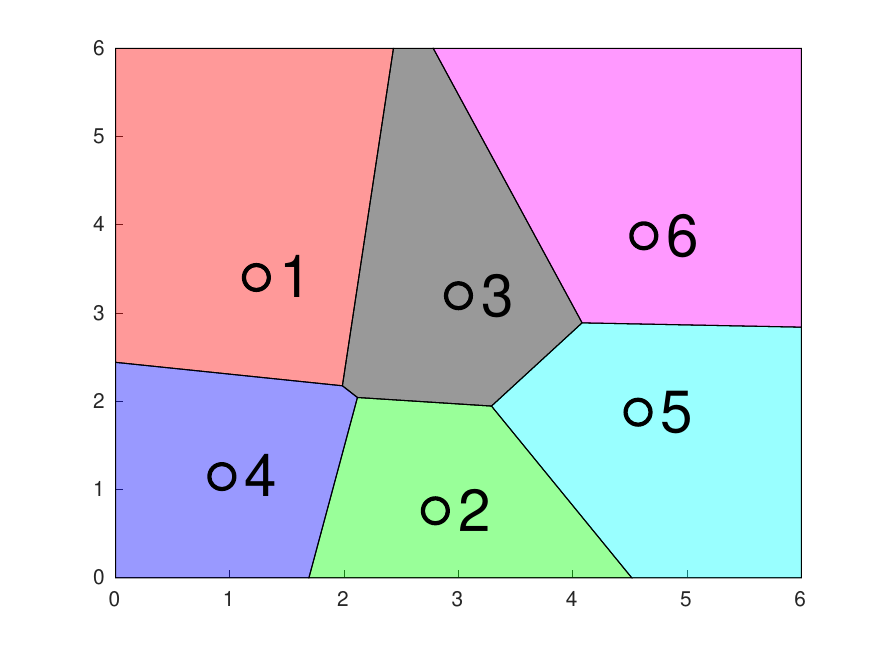} \\
        \hspace{-0.1in}\includegraphics[width=\figwidth\columnwidth]{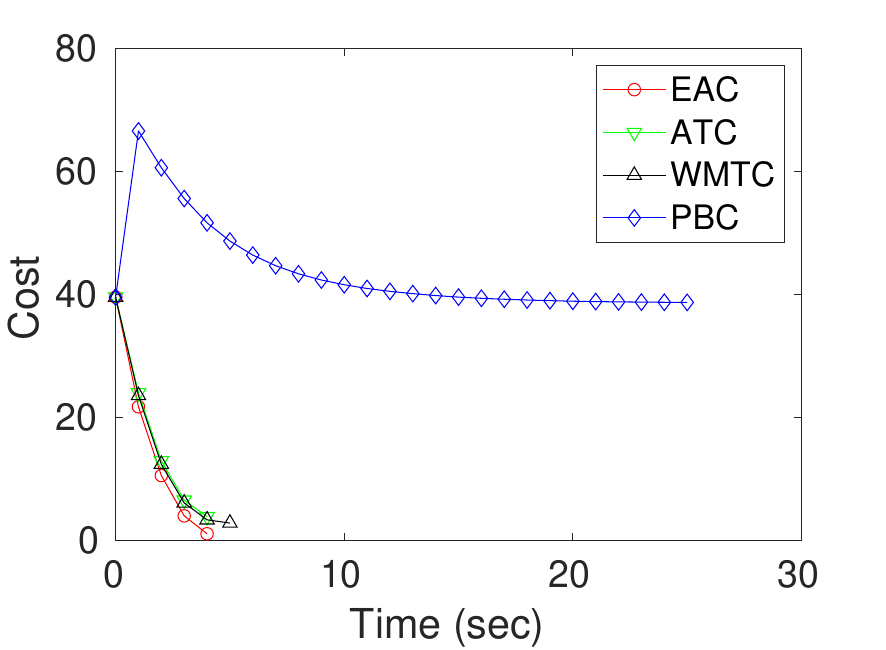} &        
        \hspace{-0.2in}\includegraphics[width=\figwidth\columnwidth]{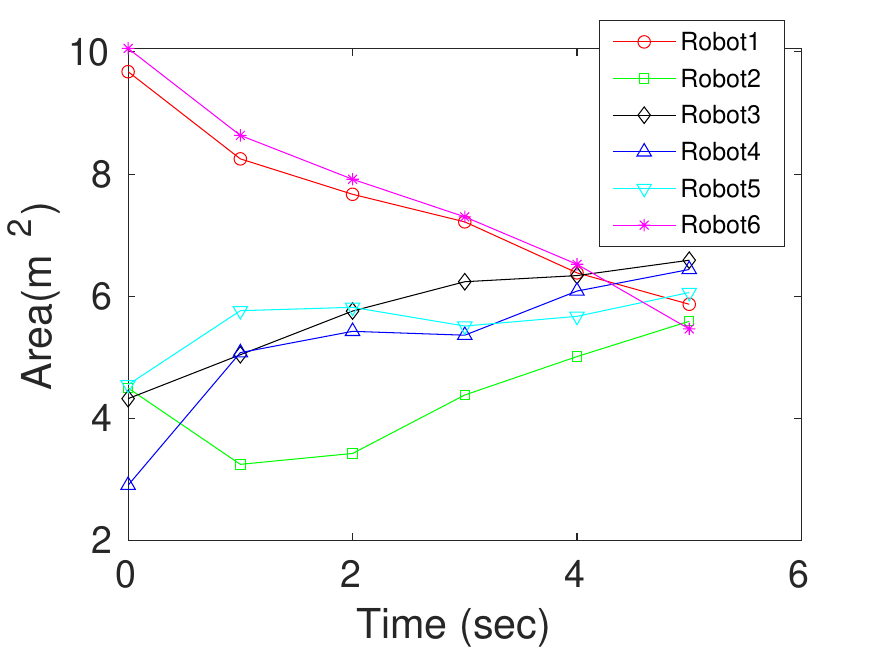} &
        \hspace{-0.2in}\includegraphics[width=\figwidth\columnwidth]{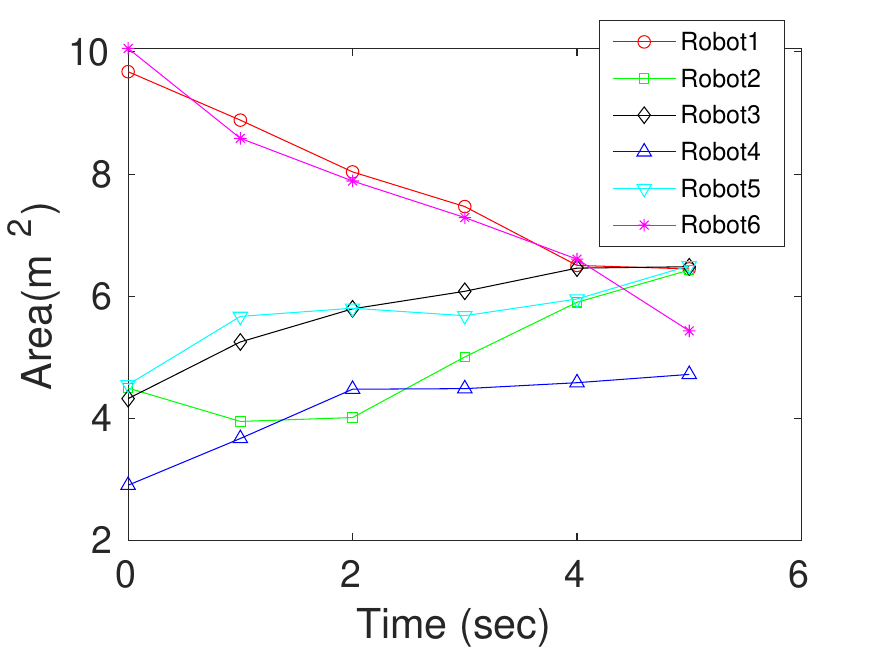} &
        \hspace{-0.2in}\includegraphics[width=\figwidth\columnwidth]{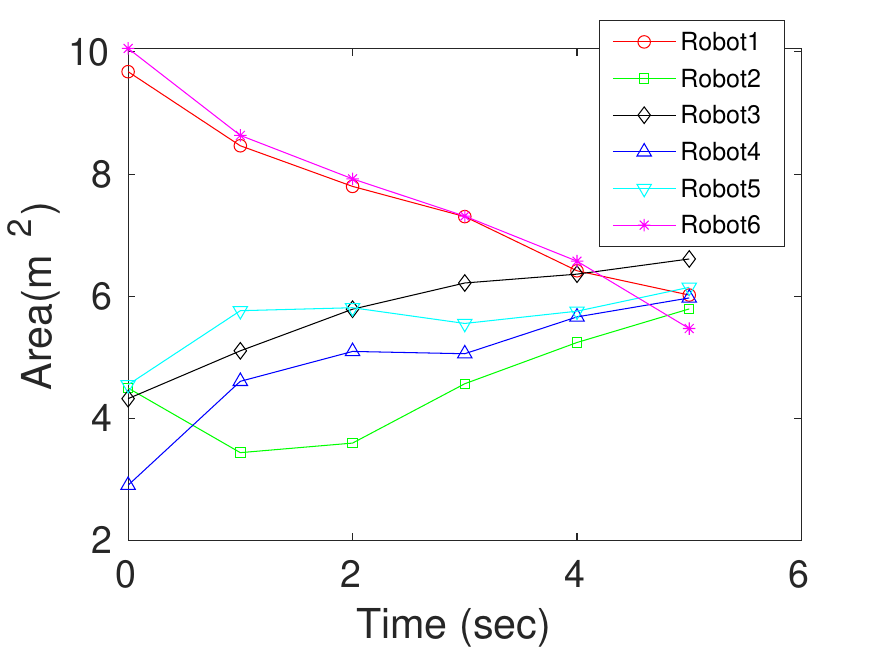} &
        \hspace{-0.2in}\includegraphics[width=\figwidth\columnwidth]{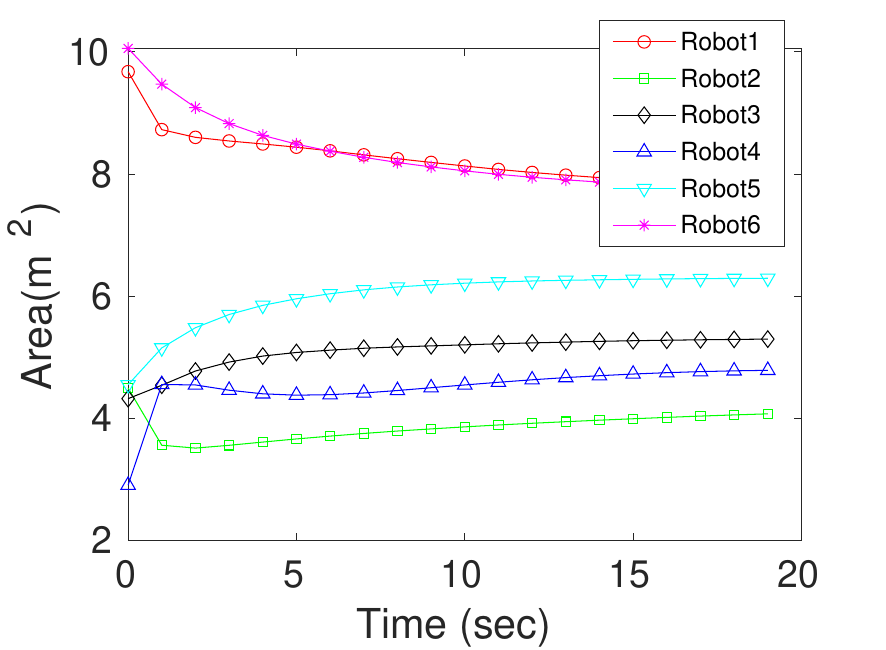}
        
    \end{tabular}
    \vspace{-4mm}
    \caption{Results of Scenario 2 experiments, where robots have a heterogeneous combination of energy levels and depletion rates $\alpha_4=4$ and $E_4^{init}=100$, where other robots have $\alpha = 1, E^{init}=25$. Because robot $4$ started with high capacity but depleted energy at a 4.4x rate compared to other robots, their dynamics would eventually cancel out each other, and the final weights should remain similar for all robots in an energy-aware coverage.
    }
    \label{fig:results_E5}
    \vspace{-2mm}
\end{figure*}

\subsection{Scenario 2 - Heterogeneous Energy Depletion} 
In this setting, the robot $4$ has the highest initial battery level (100), but its depletion rate is also the highest ($\dot{E}_4=4.4$). The exact opposite is true for all other robots, i.e., they have lower available energy (25) and depletion rate ($\dot{E}_i = 1.4$). 

The coverage area should remain similar among all robots like WMTC because a 4x higher depletion rate compensates for the 4x higher initial capacity. 
The results presented in Fig. \ref{fig:results_E5} show that only EAC was able to adjust the weights according to both the initial energy levels and depletion rates and assign almost similar weights to all the robots (refer to Table \ref{table:exp_scenarios}).  
The ATC approach could not adapt the weights effectively, and its converged weights align only with the ratio of energy depletion rates with reduced area assignment for robot $4$. On the other hand, PBC restricts the movement of all robots except $4$ and yields final weights that do not consider both available energy and depletion rates. 
The efficacy of EAC is attributed to the fact that the ratio of the updated weights corresponds well to the ratio $E_{init}/\dot{E}(t)$.

\begin{figure*}[t]
    \centering
    \begin{tabular}{ccccc}
        \textbf{\underline{Locational Cost}} & \textbf{\underline{EAC}} & \textbf{\underline{ATC}} &
        \textbf{\underline{WMTC}} & \textbf{\underline{PBC}} \\
      
        \hspace{-0.1in}\includegraphics[width=\figwidth\columnwidth]{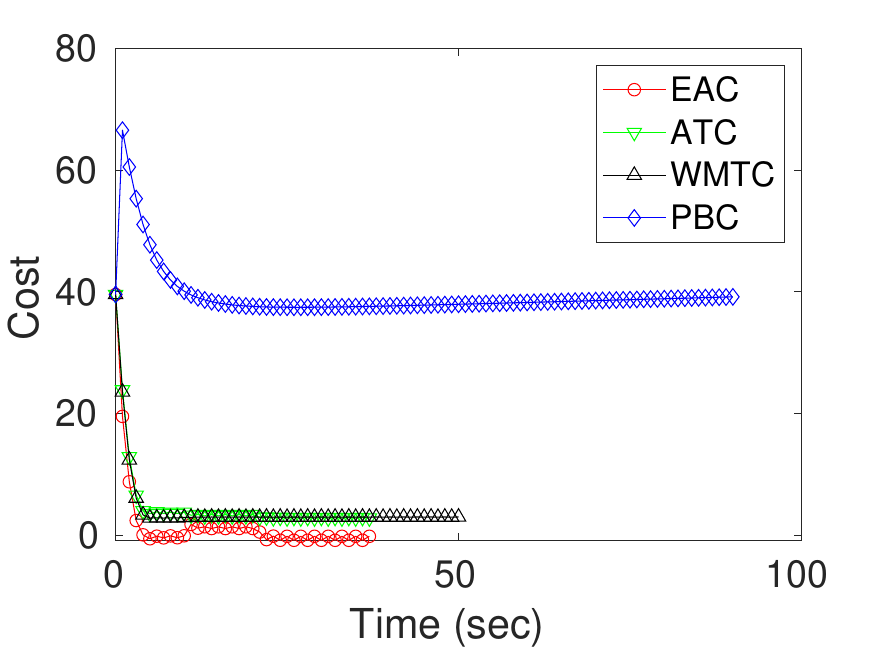} &     
        \hspace{-0.2in}\includegraphics[width=\figwidth\columnwidth]{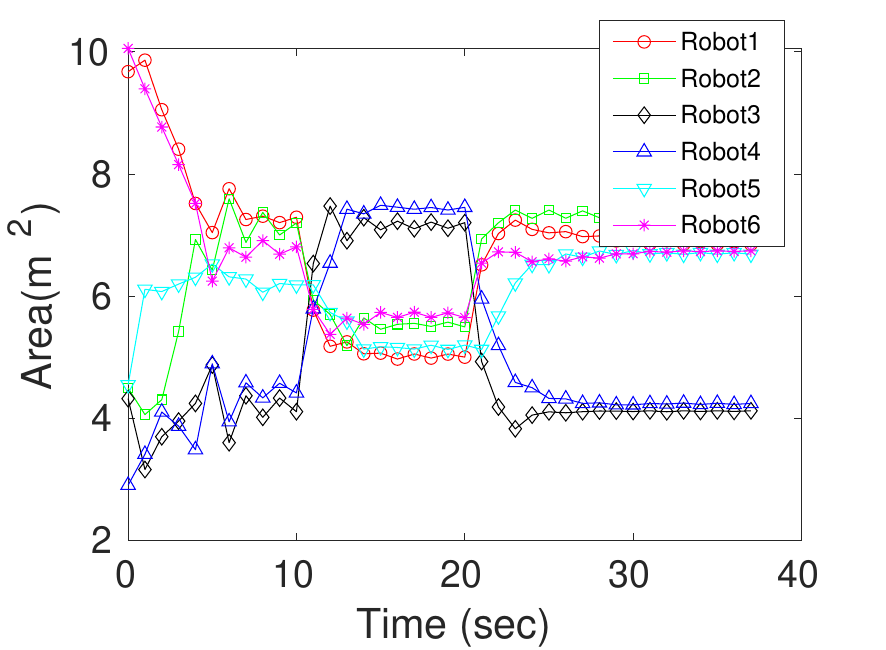} &
        \hspace{-0.2in}\includegraphics[width=\figwidth\columnwidth]{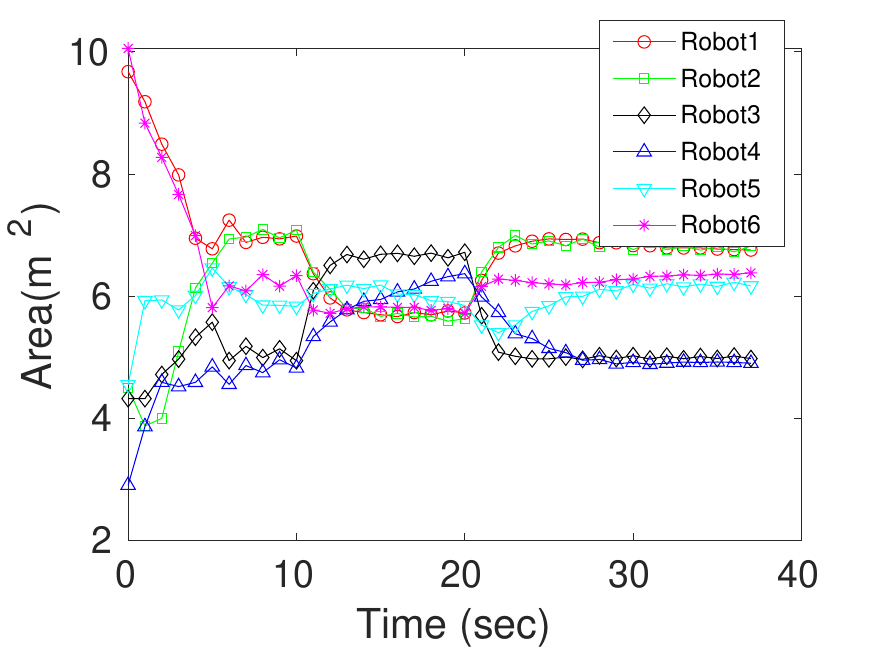} &
        \hspace{-0.2in}\includegraphics[width=\figwidth\columnwidth]{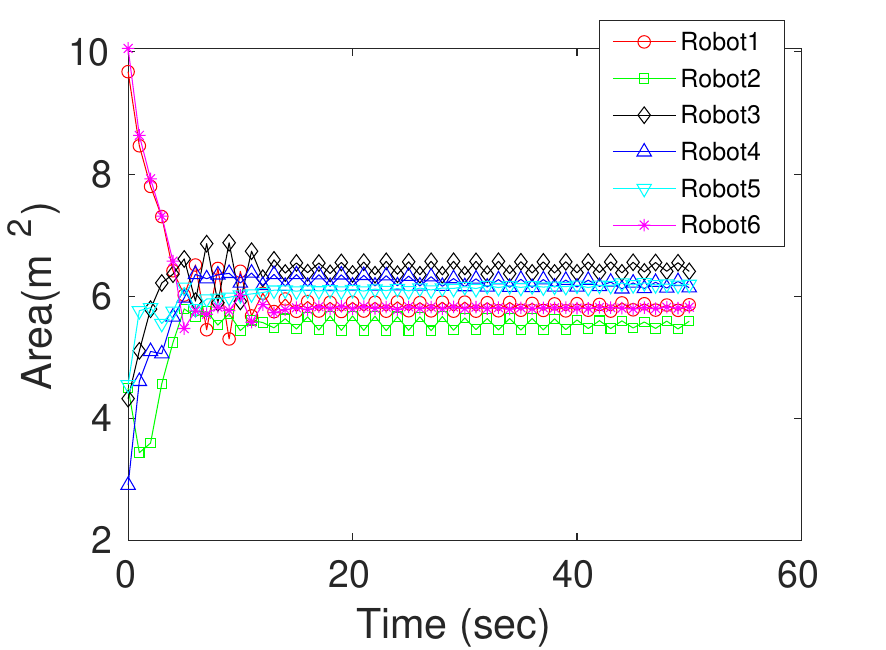} &
        \hspace{-0.2in}\includegraphics[width=\figwidth\columnwidth]{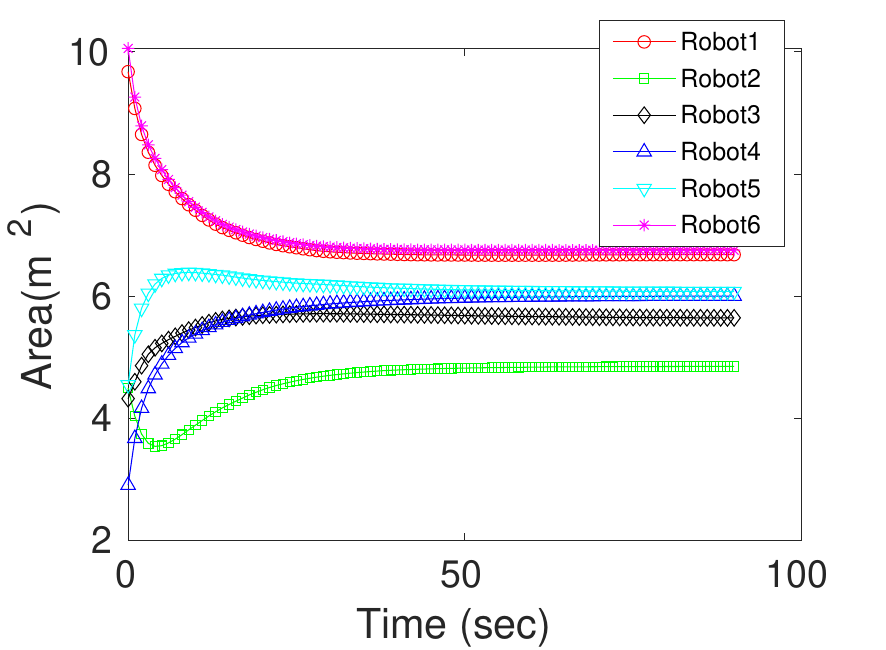}
        
    \end{tabular}
    \vspace{-4mm}
    \caption{Results of Scenario 3 experiments, showing the effect of dynamic $\beta(t)$. Initially, robots 3 and 4 had $\beta=5$, while other robots have $\beta=1$. This rate flipped at time $t=11$, and again reset to initial values at $t=22$. EAC and ATC adapted to these dynamics in energy depletion rates.}
    \label{fig:results_E6}
    \vspace{-2mm}
\end{figure*}

\subsection{Scenario 3 - Time-varying Depletion Rates}
This setting illustrates the effect of time-varying energy depletion rates by varying the $\beta$ of robots $3 \text{ and } 4$ over time. This can happen due to the terrain. 
At the start of the experiment ($t=0$), all robots have the same initial energy, the same $\alpha$, but $\beta_3=\beta_4=5$. However, at time step $t=11$, the $\beta$ values of all robots are flipped, and at $t=22$, they are reverted to their initial values as shown in Table \ref{table:exp_scenarios}. A good energy-aware controller should be able to seamlessly handle the change in energy depletion rates.

It can be observed from Fig. \ref{fig:results_E6} that the effect of the dynamics in $\beta$ value changes are perfectly captured in the evolution of weights and the coverage area yielded by the EAC controller.
Both PBC and WMTC assigned similar weights to the robots, disregarding the change in $\beta$. On the other hand, ATC also adapted its weights due to the change in $\beta$ values over time, but its effectiveness in the weight is limited to the differences in the energy depletion rates and not the ratio of the rates. At time step $t=20$, all robots once again possess the same energy value of $E_i(t)=56$, alongside the same initial depletion rate as at the start of the experiment. 
As the final iteration, we observed a replication in weight values among the robots that were evident in the initial time instants,
Additionally, EAC's cost reduction surpassed all other approaches, resulting in a substantial decrease in the locational cost value.

\subsection{Study on the Scalability and the Effect of Connectivity}
We performed several experiments to assess the proposed approach's efficacy in larger environments by varying the number of robots (along with the workspace dimensions according to the increase in robots) and varying the connectivity level of the graph $\mathcal{G}$, forming the team of robots. We executed simulations until each robot's energy exhaustion and recorded the convergence cost for various algebraic connectivity values.

The algebraic connectivity, denoted as $\lambda_2$, serves as a metric that delineates a graph's connectedness and is defined as the second-smallest eigenvalue of the Laplacian matrix associated with the graph. The convergence cost was computed as the sum of the squared difference of the weight convergence values $c_i$.
\begin{equation}
    \text{Convergence Cost} = \sum_{i=1}^n \sum_{j \in \mathcal{N}_i}^n (c_i - c_j)^2 ,
\label{convergence_cost}
\end{equation}
where $c_i = \frac{w_i \times \dot{E}_i}{E^{init}_i}$ that explains the difference in the calculated weights and the expected weights based on Eq. \eqref{eq: weightcombo_Theorem1}.

To substantiate the scalability of the approach, we simulated scenarios with \(n = 20, 50, \) and \(100\) robots in environment sizes \(50 \times 50 m^2\), \(100 \times 100 m^2\), and \(200 \times 200 m^2 \) respectively. We also varied the algebraic connectivity of the graph based on the initial positions of the robots by varying a disk radius around each robot to form its neighbor set ${\cal N}_i$. After obtaining the convergence cost for different algebraic connectivity scenarios, we use the final convergence cost corresponding to the lowest algebraic connectivity (i.e., the least connected graph with $\lambda_2$ close to 1) as a reference for comparing the convergence time. Fig.~\ref{fig:connectivity} shows the results of these experiments. The left plot shows the convergence cost of the weights for $n=20$ robots against $\lambda_2$ of the graph. The results demonstrate quick convergence when the connectivity is strong and a slower but guaranteed convergence when the connectivity of the graph is sparse. The right plot shows the time taken to converge the weights using the EAC approach for up to 100 robots at different levels of connectivity. As expected, strongly connected graphs required significantly less time to converge the weights as the graph's diameter is relatively small. This requires fewer communication rounds in Eq~\eqref{eq: weightAdaptation_EAC}, leading to quick convergence.

\begin{figure}[t]
    \centering
    \includegraphics[width=0.49\linewidth]{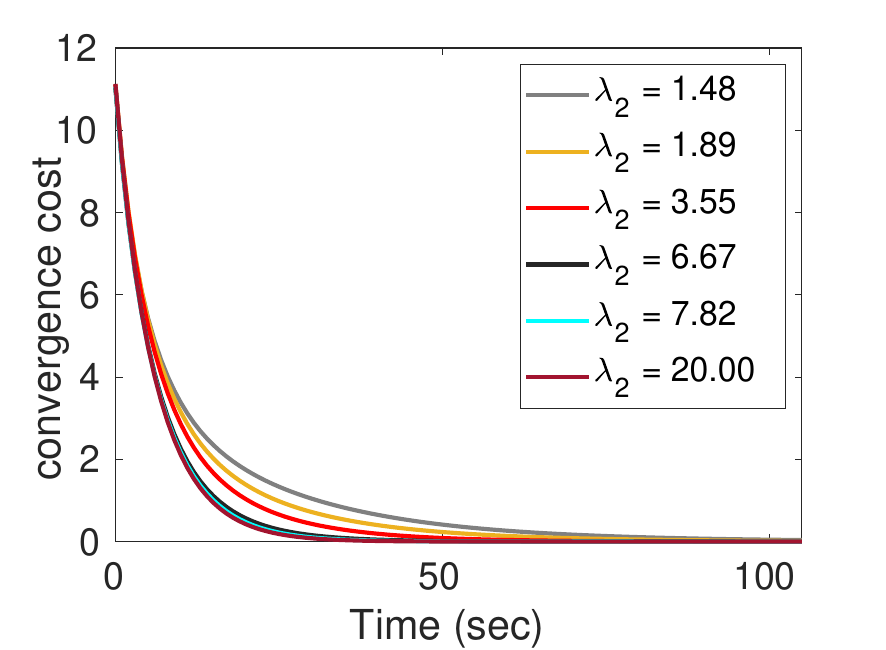}
    \includegraphics[width=0.49\linewidth]{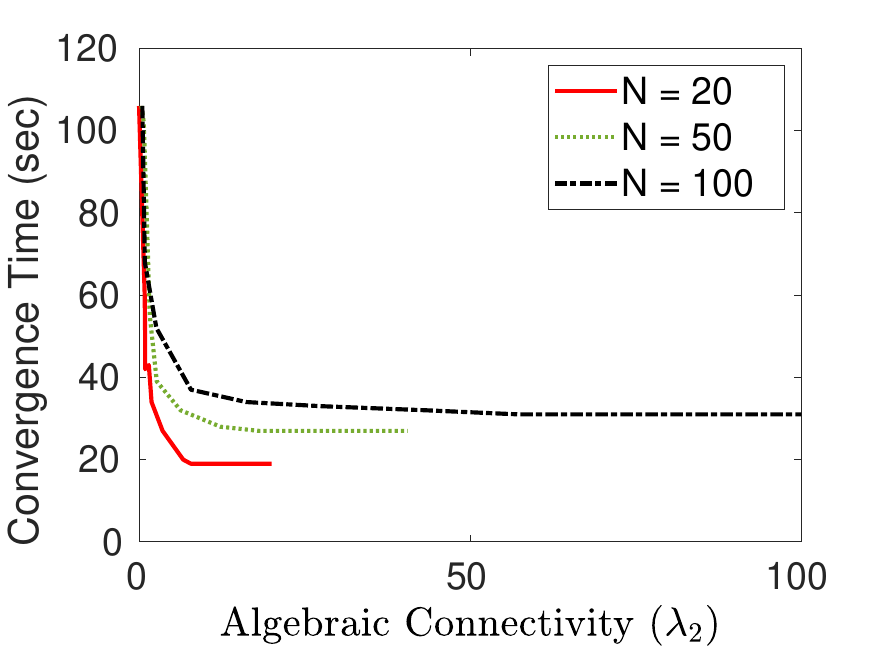}
    \vspace{-4mm}
    \caption{Results of the experiments analyzing the effect of connectivity and scalability (number of robots). The convergence cost is high for sparsely connected graphs (left), and the convergence time increases proportional to the number of robots (right), demonstrating scalability aspects.}
    \label{fig:connectivity}
    \vspace{-2mm}
\end{figure}




\begin{figure}[t]
    \centering
    \begin{tabular}{cccc}
        \textbf{\underline{Homogeneous Robots}} & & \textbf{\underline{Heterogeneous Robots}}&  \\
         \includegraphics[height=0.2\linewidth]{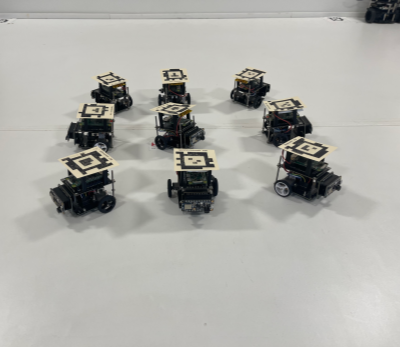} &
         \includegraphics[width=0.25\linewidth]{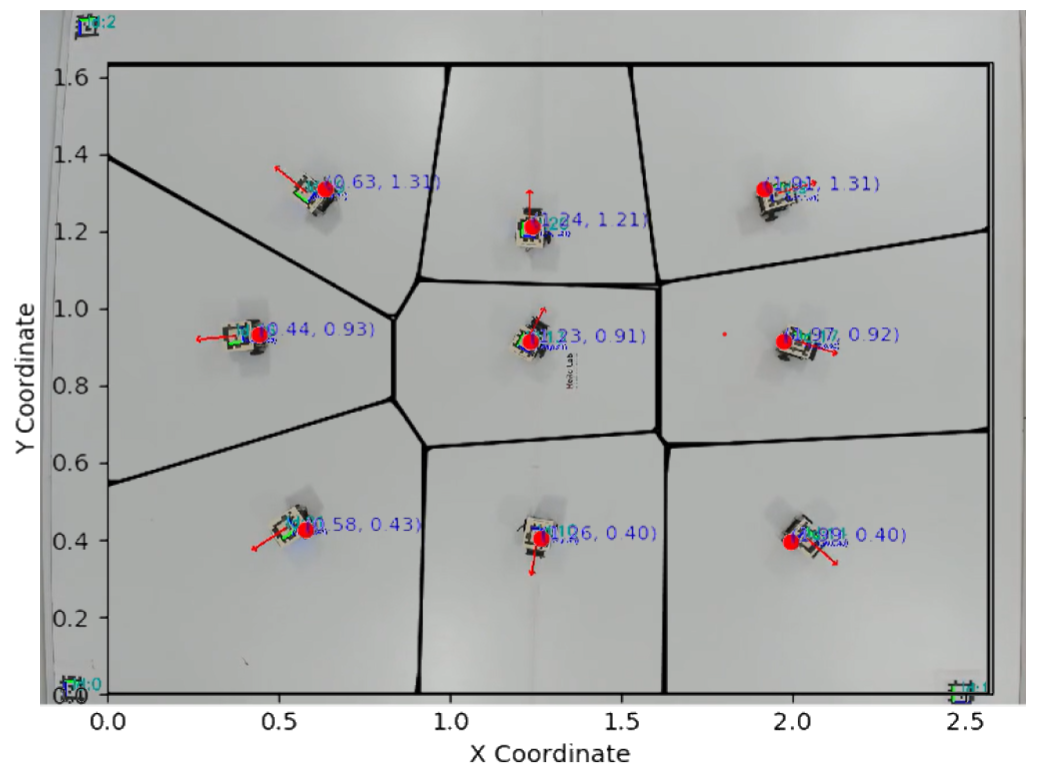} &
        \includegraphics[height=0.2\linewidth{}]{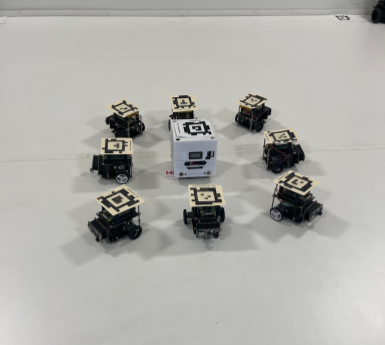} &
        \includegraphics[width=0.25\linewidth{}]{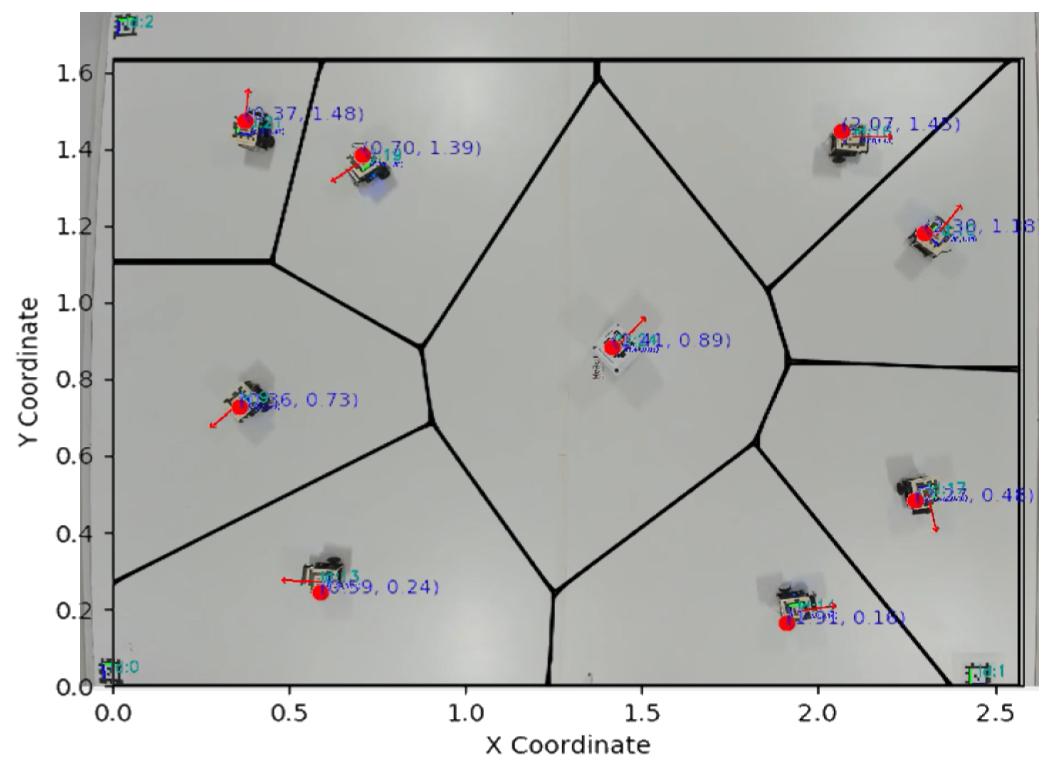} \\
    \end{tabular}
    \vspace{-4mm}
    \caption{An example outcome of running the proposed energy-aware controller on an in-house swarm robotics testbed with nine robots. The top figures show the close-up view of the robots' initial locations at the center of the workspace. In the heterogeneous case, the robot at the center has 2x higher energy capacity than the other robots. The bottom plots show their final coverage allocations, which visibly demonstrates a larger workspace for the center robot in the heterogeneous case after the convergence of coverage areas.  }
    \label{fig:heroswarm_eac}
\end{figure}

\subsection{Demonstration with Heterogeneous Swarm Robots}
Our in-house swarm robotics testbed (\(2.5 \times 2.5 m^2 \)) allows us to deploy heterogeneous types of robots in the workspace. The testbed consists of AprilTag \cite{wang2016apriltag} based position tracking and real-time robot control through ROS2 software framework \cite{macenski2022robot}. The ROS2 driver running on each robot provides real-time data on the battery voltage and multiple other sensors. As shown in Fig.~\ref{fig:heroswarm_eac} (top figures), we create two configurations with $n=9$ robots. In the homogeneous case, all robots are of the same type and have the same battery capacity and depletion rate. In the heterogeneous case, one robot is structurally larger (with $2$x battery capacity and $1.2$x depletion rate) than the other eight robots, which are of the same type. We implemented and tested our EAC approach in the swarm robot testbed with these real energy characteristics. 
Fig.~\ref{fig:heroswarm_eac} presents the results of a sample trial showing the final positions of the robots after the coverage applied in both cases. In both cases, all robots are positioned together at the center of the workspace initially, as shown in the top row of the figure. The homogeneous robot results align with the WMTC approach since similar energy characteristics would obtain the same weights for all the robots, hence assigning them almost equal area partitions. However, in the heterogeneous case, the robot weights would be adapted based on their energy characteristics, and therefore, the robot with higher battery capacity is assigned a much larger area. This effect can be observed in the right-most plot in Fig.~\ref{fig:heroswarm_eac}.
This demonstration explicates the difference in the coverage control output between a homogeneous and a heterogeneous team of robots where real energy characteristics are considered.

Additional experiments conducted on the real-world Robotarium platform \cite{wilson2021robotarium} are reported in the Appendix, where we also considered a non-uniform density function (i.e., $\phi$ in Eq~\eqref{eq:centroid-mass} is set to a function similar to the ones used in \cite{kim2022}). In addition, a supplementary video demonstrating the experiments is available\footnote{\url{https://github.com/herolab-uga/energy-aware-coverage} \\ \url{https://www.youtube.com/watch?v=HR9HCWUJz18}}.

\section{Conclusion}
Heterogeneous robots have different capabilities resulting in different energy consumption characteristics, which play a crucial role in multi-robot coverage performance.
We have proposed a novel energy-aware control law to split the workspace among $n$ robots for sensor coverage such that the weights, i.e., the areas of the robots' allocated unique sub-regions, are governed by their energy levels and the energy depletion rates. We assumed that such depletion rates might not be known a priori, and therefore, they must calculate these rates online and use that information to better partition the environment among them. Supported by theoretical results and extensive experiments in both simulations and real-world robot testbeds, we have shown that our proposed controller can allocate regions according to the robots' energy depletion rates while outperforming energy-aware controllers from the literature. In the future, we are interested in extending this energy-aware approach to informative path planning and active sensing tasks.

\include{appendix}

\bibliographystyle{plainnat}
\bibliography{ref}

\end{document}

%% file: appendix.tex
\section*{Appendix - Additional Experimental Data}

\begin{figure*}[t]
    \centering
    \begin{tabular}{ccccc}
        \textbf{\underline{Initial}} & \textbf{\underline{EAC}} & \textbf{\underline{ATC}} &
        \textbf{\underline{WMTC}} & \textbf{\underline{PBC}} \\
        \hspace{-0.2in}\includegraphics[width=\figwidth\columnwidth]{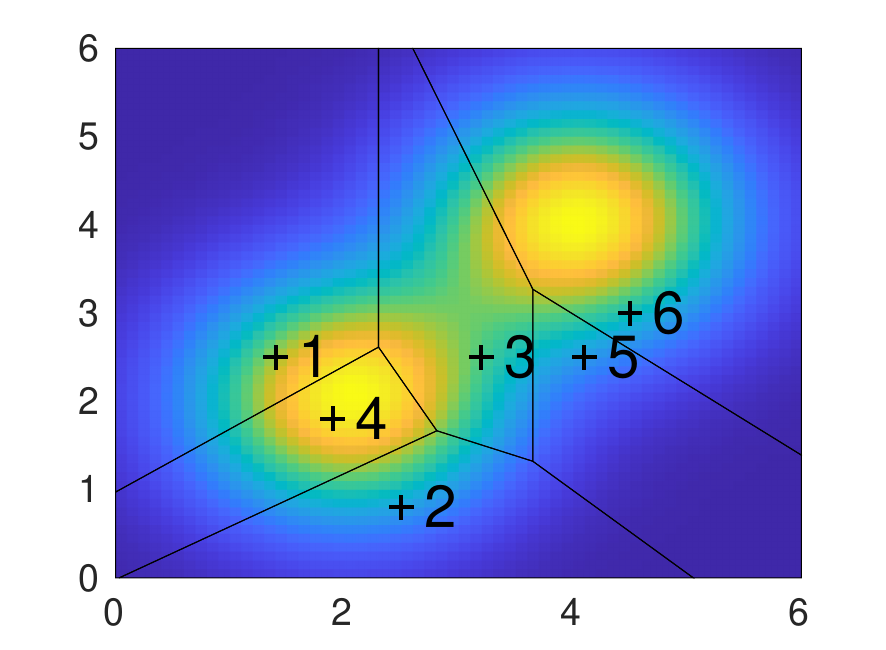} &
        \hspace{-0.2in}\includegraphics[width=\figwidth\columnwidth]{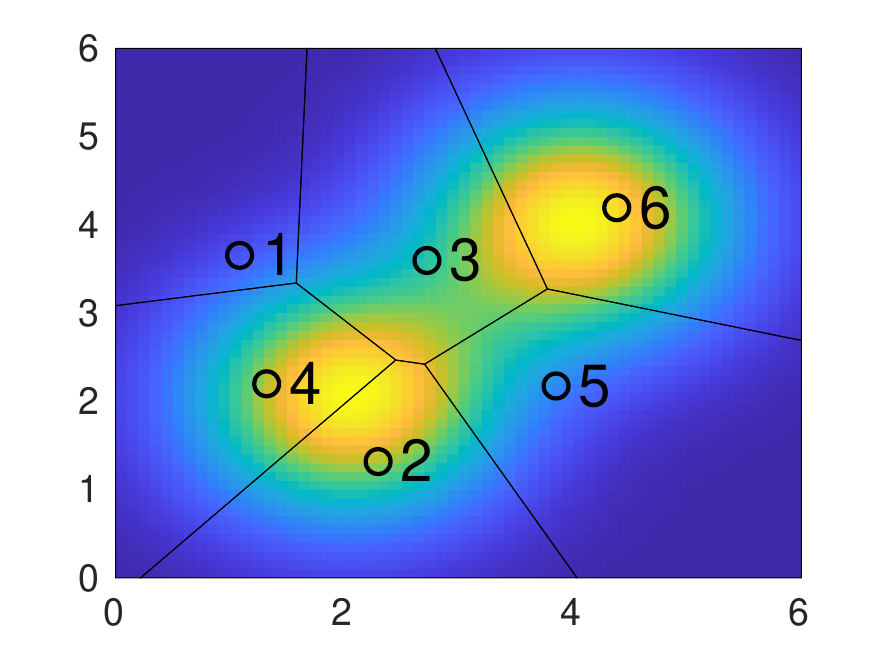} &
        \hspace{-0.2in}\includegraphics[width=\figwidth\columnwidth]{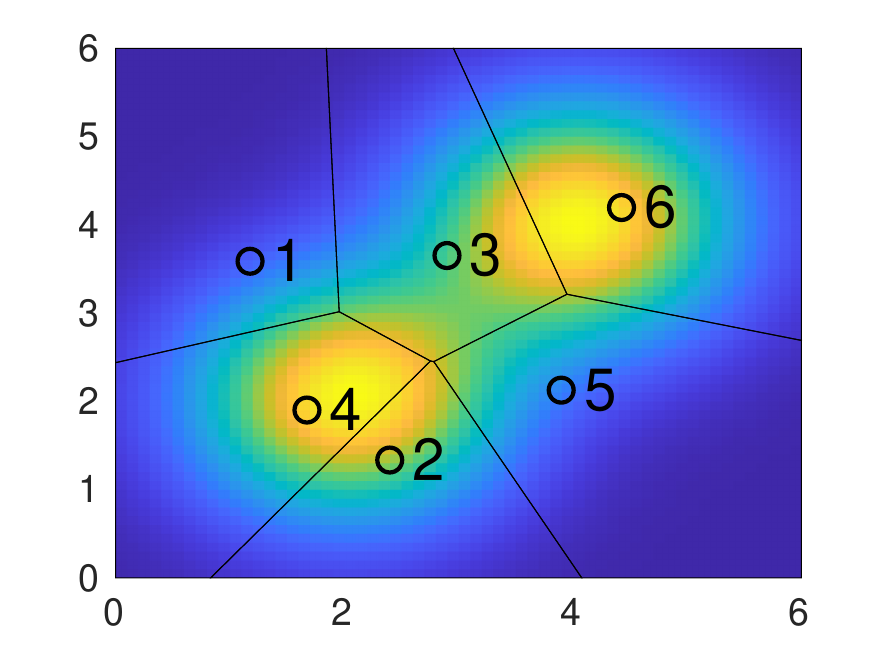 } &
        \hspace{-0.2in}\includegraphics[width=\figwidth\columnwidth]{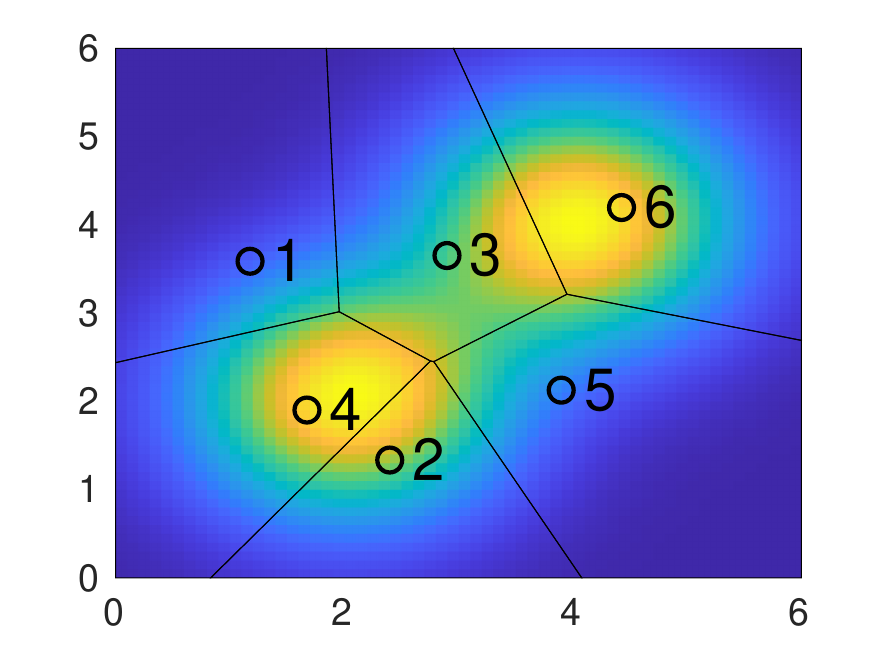} &
        \hspace{-0.2in}\includegraphics[width=\figwidth\columnwidth]{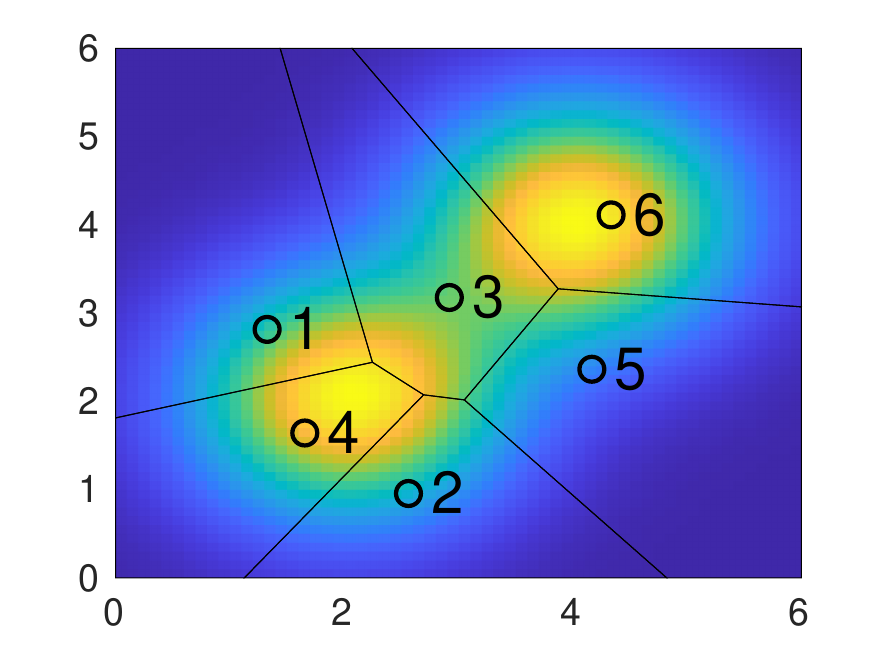} \\
        \hspace{-0.2in}\includegraphics[width=\figwidth\columnwidth]{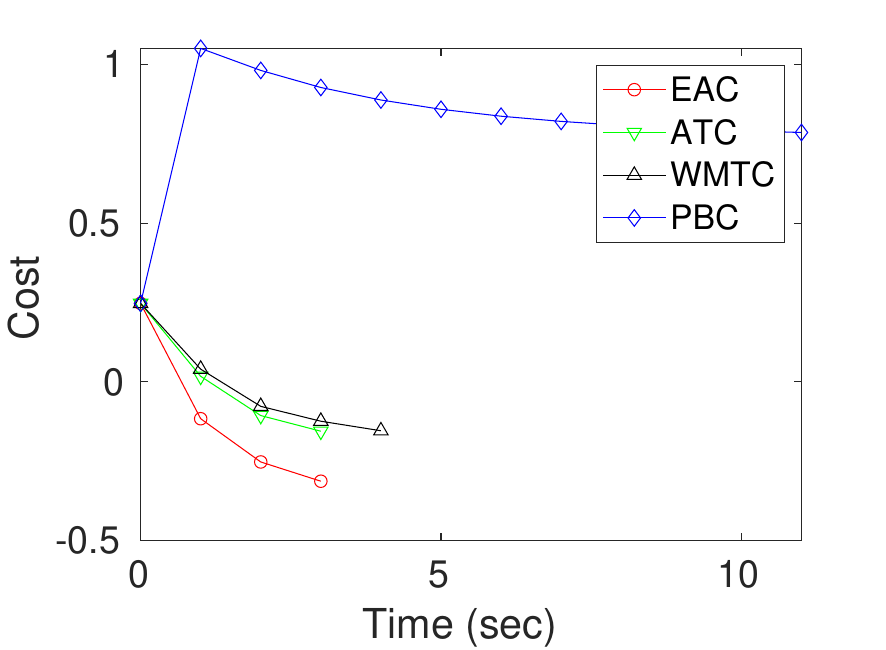} &        
        \hspace{-0.2in}\includegraphics[width=\figwidth\columnwidth]{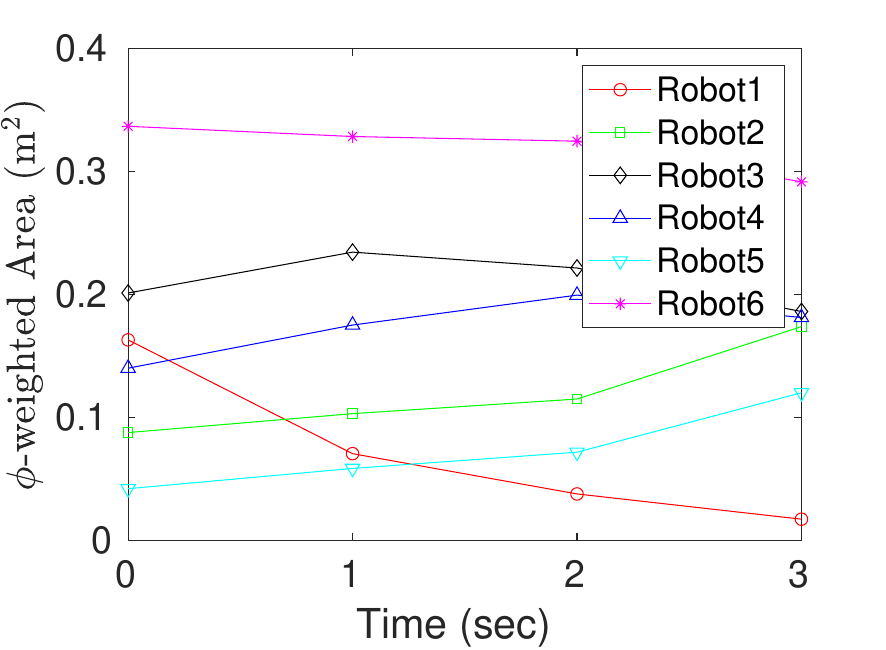} &
        \hspace{-0.2in}\includegraphics[width=\figwidth\columnwidth]{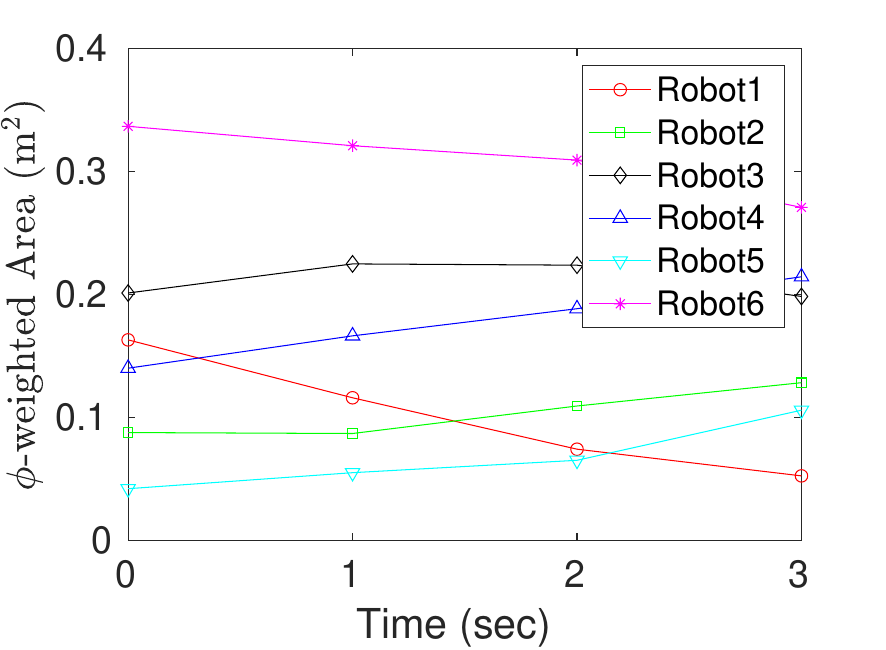} &
        \hspace{-0.2in}\includegraphics[width=\figwidth\columnwidth]{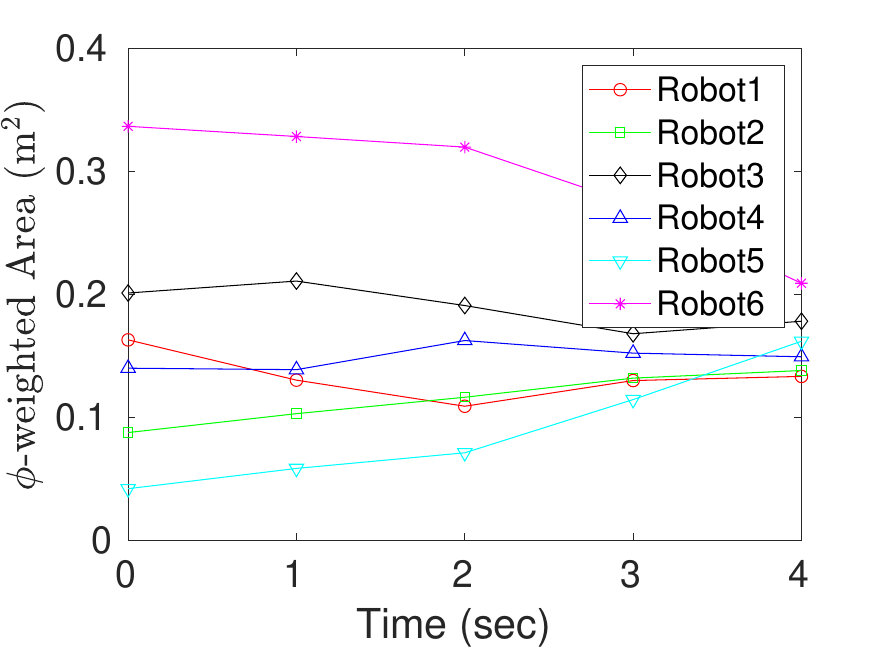} &
        \hspace{-0.2in}\includegraphics[width=\figwidth\columnwidth]{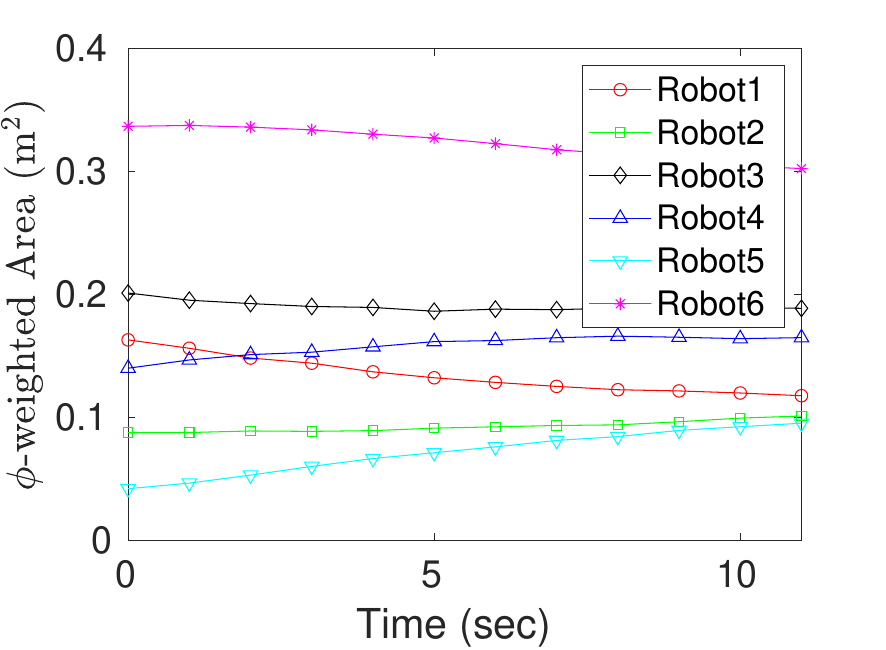}
    
    \end{tabular}
    \vspace{-4mm}
    \caption{Results of the additional experiments (non-uniform density case considering a bi-modal density function). Here, all robots have the same initial energy capacity, but robot 1 has a higher depletion rate). The WMTC coverage treats all robots equally, assigning them based on the region of importance. However, the EAC method correctly assigns robot 1 to a smaller area compared to other methods due to the depletion rate difference as expected. }
    \label{fig:results_E7}
\end{figure*}

\subsection*{Experimental Scenario with Non-uniform Density Function}
\label{sec:density-sim}
Following \cite{kim2022}, we test the proposed energy-aware coverage controller in a non-uniform density environment. A density function ($\phi(q)$ in Eq.~\eqref{eq:centroid-mass}) represents the importance of the location, e.g., representing the concentration of events or phenomena in different regions. A high density implies more effort (or time) needed to survey that point of interest, and hence the robots in high-density regions are allocated less area in the coverage. We simulate such a scenario by employing an identical density function as presented in \cite{kim2022}, with the following parameter adjustments: $\mu_1 = \begin{bmatrix} 2 \ 2 \end{bmatrix}$, $\mu_2 = \begin{bmatrix} 4 \ 4 \end{bmatrix}$, and $\Sigma = 0.9I$.
In this setting, $r_1$ has higher $\beta$ ($\beta_1=10$) i.e, its $\dot{E}_i(t)=5 $ while all other robots have  $\dot{E}_i(t)=1.4$, however the initial energy and $\alpha$ of all robots remain the same. 
The results are presented in Fig. \ref{fig:results_E7}. Even with a non-uniform density function, the EAC effectively adjusted the weights of robots, resulting in an approximate $50\%$ decrease in cost ($-0.31$) as compared to WMTC ($-0.15$) and ATC ($-0.16$). Furthermore, the EAC demonstrated a remarkable achievement by attaining $140\%$ lower cost than PBC ($0.78$).

\subsection*{Demonstration in the Robotarium Hardware Testbed}

 \begin{figure*}[t]
     \centering
     \begin{tabular}{ccccc}
         \textbf{\underline{Initial}} & \textbf{\underline{EAC}} & \textbf{\underline{ATC}} &
         \textbf{\underline{WMTC}} & \textbf{\underline{PBC}} \\
         \hspace{-0.05in}\includegraphics[width=0.19\textwidth]{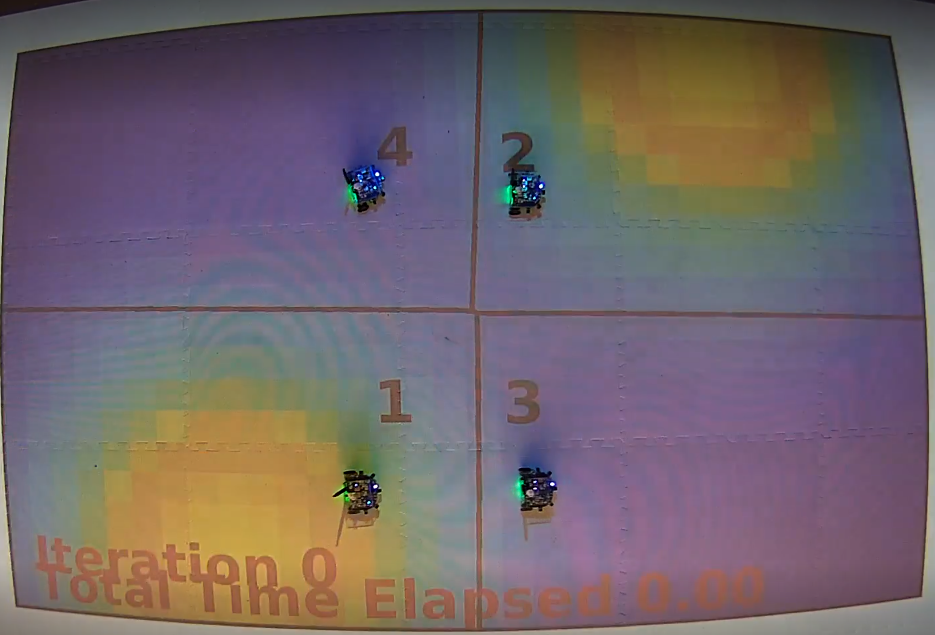} &
         \includegraphics[width=0.19\textwidth]{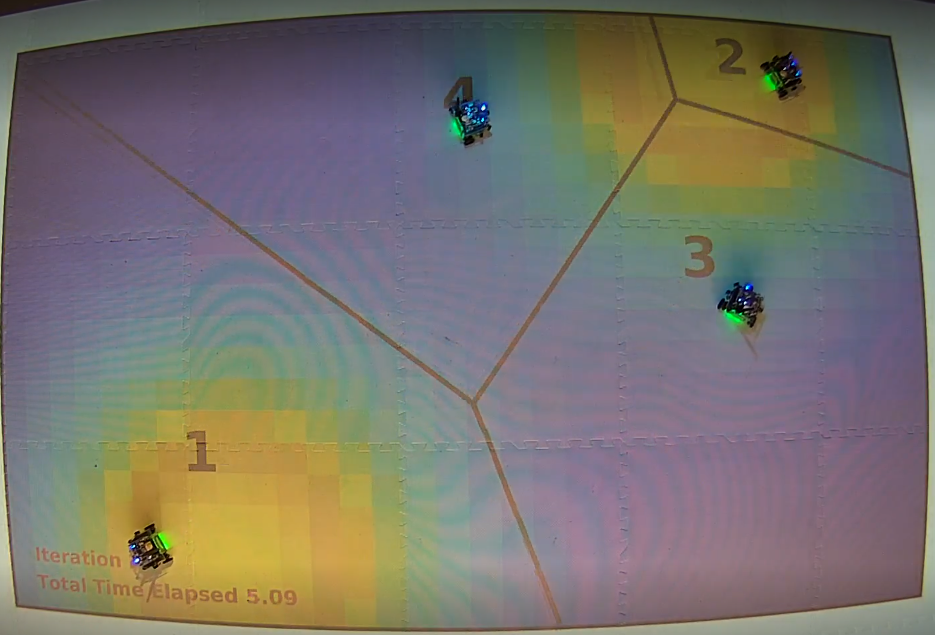} &
         \includegraphics[width=0.19\textwidth]{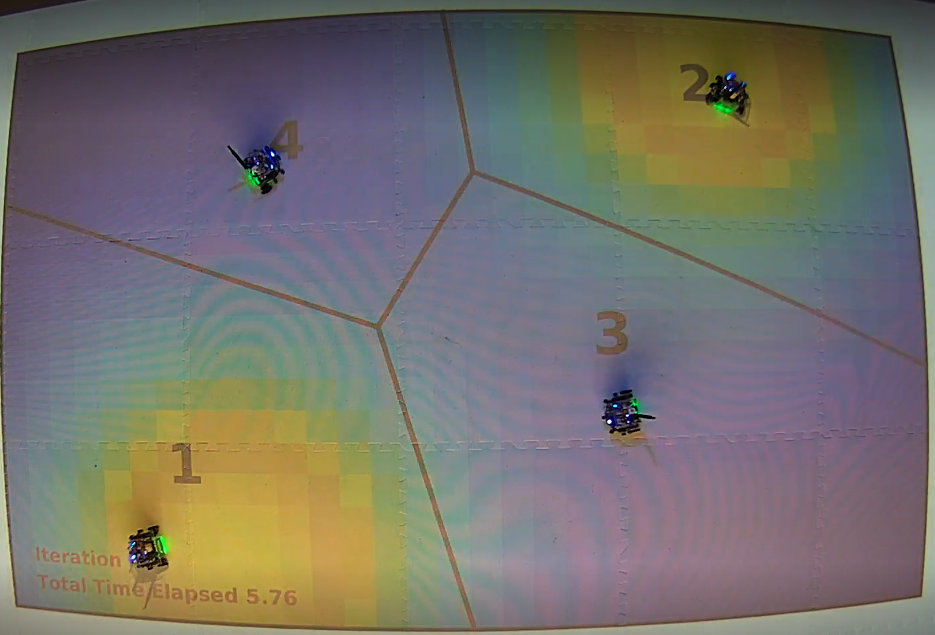 } &
         \includegraphics[width=0.19\textwidth]{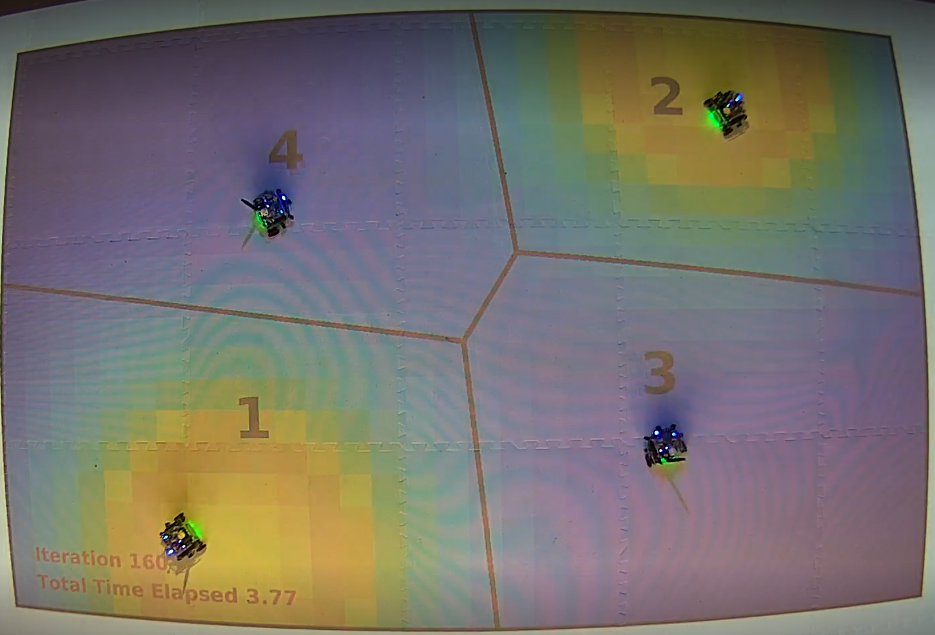} &
         \includegraphics[width=0.19\textwidth]{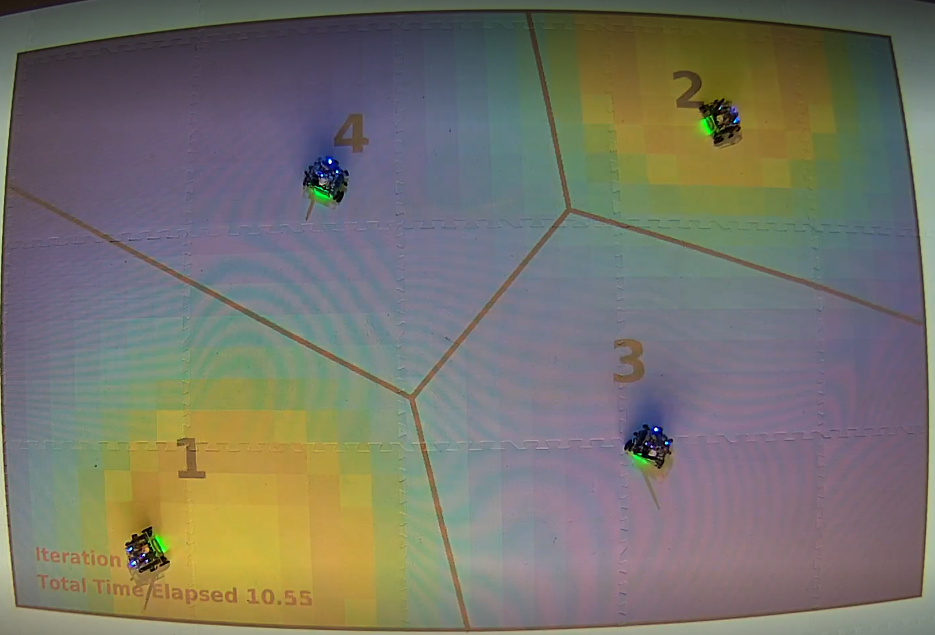} \\
     \end{tabular}
     \vspace{-4mm}
     \caption{Real-robot experiment in Robotarium with a bi-modal density function. The robots have similar energy characteristics, but robot $2$ (top right) has low initial energy and is close to a source. Both these effects assign significantly less weight (and coverage) to robot $2$ in EAC than other controllers. }
     \label{fig:results_E8}
 \end{figure*}
 
To validate the performance of the proposed controller in the real world, we showcase experiments in two different robot testbeds. The Robotarium's hardware testbed \cite{wilson2021robotarium} allows us to verify the control accuracy of robots with simulated energy characteristics.

In this demonstration, we used four robots, and robot $2$ has a lower initial energy of $E_{init} = 70$ compared to all other robots ($E_{init}= 100$), while their depletion rates are the same. 
 We also simulated a bi-modal density function $\phi(q)$ using the same density function mentioned in Scenario 5 (Sec. \ref{sec:density-sim}) with two sources closer to robots $1$ and $2$.
 We have employed an identical density function as presented in \cite{kim2022}, with parameter adjustments: $\mu_1 = \begin{bmatrix} 2 \ 2 \end{bmatrix}$, $\mu_2 = \begin{bmatrix} 4 \ 4 \end{bmatrix}$, and $\Sigma = 0.9I$.
 We expect the controller to learn the difference between available energy and repartition accordingly. 
 i.e., we expect robot $2$ to have a significantly reduced coverage area for two reasons: 1) it has a lower energy capacity, and 2) it is in a region of high importance (high density), which needs to be covered more carefully, further necessitating its weight reduction.

 The results for this experiment are presented in Fig \ref{fig:results_E8}. 
 Although ATC and PBC assigned a smaller area to robot $2$ (primarily due to the influence of the density values), EAC was more effective in allocating a significantly smaller area to it.
 More results from the simulations and the real-world experiments are included in the attached video.
As expected, the bi-modal density function has a source closer to $r_2$ and this further contributed to the reduction in the weights for $r_2$ in EAC. 
Consequently, this increased operational efficiency extends the operational time of $r_2$, thereby optimizing the overall network lifetime of the robots. 
Together with the Robotarium's hardware experiment data, we obtain more evidence that the proposed energy-aware coverage controller is useful in practical real-world applications.
